%% file: neurips_2026.tex
\documentclass{article}
\usepackage{comment}
\PassOptionsToPackage{numbers, compress}{natbib}
\usepackage[preprint]{neurips_2026}

\usepackage[utf8]{inputenc} 
\usepackage[T1]{fontenc}    
\usepackage{url}            
\usepackage{booktabs}       
\usepackage{amsfonts}       
\usepackage{nicefrac}       
\usepackage{microtype}      
\usepackage{xcolor}         

\usepackage{amsmath}
\usepackage{amssymb}
\usepackage{mathtools}
\usepackage{amsthm}

\usepackage{microtype}
\usepackage{graphicx}
\usepackage{subcaption}
\usepackage{booktabs} 
\usepackage{xcolor}
\usepackage{colortbl}
\usepackage{makecell}

\usepackage{amsthm}
\theoremstyle{plain}
\newtheorem{theorem}{Theorem}[section]
\newtheorem{proposition}[theorem]{Proposition}
\newtheorem{lemma}[theorem]{Lemma}

\theoremstyle{definition}

\newtheorem{assumption}[theorem]{Assumption}
\theoremstyle{remark}

\usepackage[utf8]{inputenc}
\usepackage{tikz}
\usepackage{cancel}
\usepackage{enumitem}
\usepackage{algorithm}
\usepackage{algorithmic}

\usetikzlibrary{positioning, shapes.geometric, arrows.meta, calc}
\usepackage[most]{tcolorbox}
\usepackage{xspace}
\definecolor{mycitecolor}{RGB}{0, 102, 204}
\usepackage[colorlinks,linkcolor=red,citecolor=mycitecolor,filecolor=magenta]{hyperref}
\usepackage{cleveref}

\newtcolorbox{myframe}[1][]{
  enhanced,
  colback=white,              
  colframe=cyan!30!black,    
  boxrule=0.8pt,              
  arc=2pt,                   
  left=6pt,right=6pt,top=6pt,bottom=6pt,
  #1
}
\newcommand{\algfullname}{Guided Denoiser Self-Distillation\xspace}
\newcommand{\algname}{GDSD\xspace}
\usepackage[export]{adjustbox}

\title{GDSD: Reinforcement Learning as Guided Denoiser Self-Distillation for Diffusion Language Models}

\author{%
  Xiaohang Tang$^{*}$ \\
  UCL Dept. of Statistical Science \\
  UCL Centre for AI  
  \And Keyue Jiang$^{*}$ \\
  Alibaba Group \\
  UCL  Centre for AI \& Dept. of EEE 
  \And Che Liu \\
  Imperial College London
  \And Qifang Zhao \\
  Alibaba Group
  \And Xiaoxiao Xu \\
  Alibaba Group
  \AND Sangwoong Yoon \\
  UNIST
  \And Ilija Bogunovic$^{\dagger}$ \\
  University of Basel
}

\begin{document}

\maketitle

\begin{abstract}
Reinforcement learning (RL) can be used to improve the policy (denoiser) of diffusion large language models (dLLMs), while being hindered by the intractability of the policy likelihood.
A dominant and efficient family of methods replaces the likelihood in standard RL with its evidence lower bound (ELBO), estimated from randomly masked sequences.
Despite being well aligned with pre-training, these approaches introduce bias through training--inference mismatch by using the ELBO as a likelihood surrogate, which can degrade performance.
In this work, we propose \textbf{\algfullname (\algname)} to directly distill the denoiser of dLLMs from an advantage-guided self-teacher, derived from the closed-form optimum of reverse-KL regularized RL.
GDSD matches the dLLM's denoiser logits to the teacher's via a normalization-free objective, which reduces RL to likelihood-free self-distillation and thus bypasses the TIM biases.
Recent ELBO-based methods emerge as instances of applying different distillation divergences, but with diagnosable pathologies that GDSD avoids.
On planning, math, and coding benchmarks with LLaDA-8B and Dream-7B, GDSD consistently outperforms prior state-of-the-art ELBO-based methods with a more stable training reward dynamics, achieving test-accuracy improvements of up to $+19.6\%$. 
These results suggest that direct denoiser self-distillation, without relying on an ELBO likelihood surrogate, can provide a more stable and effective RL procedure for dLLMs. Code is available at \url{https://github.com/GaryBall/GDSD}.
\end{abstract}

\section{Introduction}
\label{intro}


Diffusion Large Language Models (dLLMs) have emerged as efficient alternatives to autoregressive models (ARMs). dLLMs generate multiple tokens in a single decoding step and do not follow a strictly left-to-right generation order, thereby improving generation efficiency and unlocking token dependencies beyond the left-to-right paradigm in ARMs. This promise is also reflected in several recent releases, which have scaled model size while substantially improving inference efficiency. Open-weight reasoning dLLMs have grown from 8B-parameter models \citep{nie2025large,zhu2025llada} to 100B parameters in LLaDA 2.0 \citep{bie2025llada2}, with inference reported to be more than $3\times$ faster than that of even smaller ARMs \citep{bie2026llada2}. Closed models such as Mercury similarly report speedups of up to $10\times$ relative to ARMs \citep{labs2025mercuryultrafastlanguagemodels}. Despite these efficiency gains, dLLMs still lag behind the state-of-the-art ARMs in generation quality, highlighting the need for effective fine-tuning methods, such as reinforcement learning (RL).

The central obstacle to RL with dLLMs is that the policy likelihood \emph{is intractable}. Two families of methods have emerged as strong-performing solutions. RL based on \textbf{trajectory likelihood} estimates the exact likelihood of the trajectory of a sequence in the reverse process by accumulating the transition probabilities \citep{huang2025reinforcing,wang2025revolutionizing,chen2025dultra,turok2026duel}. While unbiased in principle, these methods incur substantial computational cost per gradient step, being misaligned with the pre-training objectives. RL based on \textbf{sequence likelihood} applies likelihood evidence lower bound (ELBO) as a surrogate, estimated from randomly masked sequences \citep{zhao2025d1,yang2025mmada,zhu2025llada,tang2025wd1,ou2025principled,wang2025spg}. 
This approach is computationally efficient, and naturally aligned with the objective of current dLLMs pre-training, which is itself also defined on a randomly masked sequence. Thus, we focus on the sequence likelihood family in this work.

RL based on sequence likelihood with ELBO surrogate is inherently off-policy, yet existing methods attempt to correct it with importance sampling ratio based on ELBO. These methods naively plug the ELBO surrogate into policy gradient or PPO-style objectives in the place of true likelihood, introducing a well-known bias called training-inference mismatch (TIM) \citep{qi2025defeating,liu-li-2025-rl-collapse}. First, the non-negligible gap between the ELBO estimate and the true likelihood \citep{jiang2026diffusion} biases the importance ratio, which has been shown to degrade performance \citep{jiang2026diffusion}, and can even cause training collapse \citep{zhong2026stabilizing}. Additionally, dLLM rollouts are produced by iterative re-masking and block-wise decoders \citep{ye2025dream,nie2025large,arriola2025block}, whose induced sampling distribution differs from the training policy (i.e., ELBO). The importance-sampling formulation requires a tractable training-sampling relationship that ELBO does not provide, which motivates a different formulation.

To address the TIM bias, we shift from importance-sampling RL with the ELBO surrogate to direct, off-policy self-distillation on the denoiser. To our knowledge, this is the first principled method recasting RL for dLLMs as self-distillation to avoid TIM bias by design. Our contributions are:
\vspace{-.5em}
\begin{itemize}[leftmargin=1em]
    \item We reduce RL with dLLMs to denoiser self-distillation: under reverse-KL regularized RL, the closed-form optimal policy induces a guided denoising distribution, an advantage guided denoiser as self-teacher, which we aim to distill at each RL step. (Section \ref{sec:reformulation})
    \item  We propose \textbf{\algfullname (\algname)}, a squared-logit distillation loss with a normalization-free reformulation that eliminates the partition function via logit centralization. Distillation allows for off-policy update, thus bypassing TIM bias entirely (Sections \ref{sec:gdsd}, \ref{sec:norm_free})
    \item We provide insights on how existing ELBO-based methods are connected to alternative divergence choices under our distillation framework. We demonstrate that our squared-logit instance avoids both the data inefficiency of weighted ELBO instances (wd1 \citep{tang2025wd1}, DMPO \citep{zhu2025enhancing}) and the TIM bias of policy gradient instances (SPG \citep{wang2025spg}, UniGRPO \citep{yang2025mmada}, ESPO \citep{ou2025principled}). (Section \ref{sec:connection})
    \item  In experiments compared to prior state-of-the-art ELBO-based methods, GDSD shows more stable training reward dynamics. On planning tasks, GDSD with Dream-7B achieves test-accuracy gains of up to $+19.6\%$. With LLaDA-8B, GDSD also consistently improves performance across planning, math, and coding benchmarks, with gains ranging from $+0.6\%$ to $+5\%$. (Section \ref{sec:exp})
\end{itemize}

\input{neurips2026/preliminary}

\input{neurips2026/method}

\input{neurips2026/experiments}
\bibliography{ref}
\bibliographystyle{plainnat}


\input{neurips2026/appendix}



\end{document}

%% file: neurips2026/preliminary.tex
\section{Preliminaries}

In this section, we introduce the formulation of current diffusion large language models (dLLMs), namely masked diffusion models, and their reinforcement learning formulation.

\subsection{Masked Diffusion Models}
\label{sec:mdm}
Masked Diffusion Models (MDMs) \citep{sahoo2024simple} represent a prominent class of non-autoregressive generative models that operate over a discrete state space. Let $\mathcal{V}$ denote the categorical vocabulary and $[\text{M}]$ a dedicated mask token, such that the augmented space is $\mathcal{V}' = \mathcal{V} \cup \{[\text{M}]\}$. 

For a \textit{clean sequence} with length $N$, denoted by $x_0 \in \mathcal{V}^N$, the forward process $q(x_t | x_0)$ defines a continuous-time transition for $t \in (0, 1]$, where \textit{masked sequence} $x_t \in \mathcal{V'}^N$. Forward transition from state $x_s$ to $x_t$ ($0 \leq s < t$) is formulated as: $q(x_t | x_s) = \mathrm{Cat}(x_t; Q(s,t)^\top x_s)$ where $Q(s,t)$ is the transition kernel. The MDM, denoted by $p_\theta$, is then trained to model the reverse process. In particular, we term the reverse model for clean sequence prediction (i.e. $p(x_0 | x_t)$) as \textit{denoiser} in this work. 

Let $x_0^{(n)}$ denote the $n$-th token of the clean sequence, open large-scale MDMs, such as LLaDA-8B \citep{nie2025large} and Dream-7B \citep{ye2025dream}, directly model the token-level denoising distribution $p_\theta(x_0^{(n)} | x_t)$, and are pre-trained by \textit{maximizing} the following negative cross-entropy objective \citep{ou2025absorbingdiscretediffusionsecretly,shi2024simplified}:
\begin{align}
{L}(x_0^{(n)}; p_\theta) = \mathbb{E}_{t \sim \mathcal{U}[0,1],\ x_t \sim q(x_t | x_0)} 
\left[
w(t) \cdot \mathbf{1}(x_t^{(n)} = [\text{M}]) \log p_\theta(x_0^{(n)} | x_t)
\right].
\label{eq:md_elbo}
\end{align}
The sequence-level objective is obtained by accumulating the token-level objective, specifically
\(
{L}(x_0; p_\theta) = \sum_{n=1}^{N} {L}(x_0^{(n)}; p_\theta),
\)
which is equivalent to the likelihood evidence lower bound (ELBO). 
\Cref{eq:md_elbo}, the token-level contribution to the ELBO \citep{ou2025principled}, has also been used to approximate the token-level likelihood in reinforcement learning (RL) \citep{yang2025mmada,gong2025diffucoder}.

\subsection{RL for Masked Diffusion Models}


The goal of RL is to maximize a reward function $r$, which may be either a verifiable reward \citep{shao2024deepseekmath} or the output of a pre-trained reward model \citep{ouyang2022training}. Let $\pi$ denote policy, being the conditional probability of generating clean completion $x_0$ conditioned on prompt $c$ \footnote{In the rest of this paper, we omit the prompt notation $c$ in policy to abbreviate $\pi(x_0|c)$ as $\pi(x_0)$ for simplicity.}. Let $A(x_0)$ denote the advantage function, abbreviated as $A$. Advantage can be estimated with rewards: $A(x_0)=r(x_0)- \mathbb{E}_{\pi_\text{old}}[r(x_0)]$ \citep{shao2024deepseekmath,liu2025understanding}. The policy gradient (PG) and Proximal Policy Optimization (PPO) objectives are defined as
\begin{align}
\mathcal{L}_{\text{PG}}=\mathbb{E}_{ x_0 \sim \pi_{\text{old}}}
\big[
\frac{{\pi}_\theta(x_0)}{{\pi}_\text{old}(x_0)} A \big],\ \mathcal{L}_{\text{PPO}}=\mathbb{E}_{ x_0 \sim \pi_{\text{old}}}
\bigg[
\sum_{n=1}^N \min\big(\frac{{\pi}_\theta(x_0^{(n)})}{{\pi}_\text{old}(x_0^{(n)})} A, \text{clip}(\frac{{\pi}_\theta(x_0^{(n)})}{{\pi}_\text{old}(x_0^{(n)})}, 1\pm \epsilon) A \big)\bigg]. \nonumber
\end{align}
However, in dLLMs, both the token-level likelihood $\pi_\theta(x_0^{(n)})$ and the sequence-level likelihood $\pi_\theta(x_0)$ are intractable, making the direct application of RL difficult. \citet{zhao2025d1} propose a simplified likelihood approximation by fully masking the completion meanwhile randomly masking the prompt, but this has been shown to underperform ELBO-based methods \citep{ou2025principled}.

\subsection{RL with Likelihood Surrogate ELBO}
\label{sec:tim}
A widely adopted family of methods estimates ELBO (\Cref{eq:md_elbo}) and uses it as a surrogate data likelihood for RL  \citep{gong2025diffucoder,yang2025mmada,ou2025principled,bie2026llada2}. In each RL step, after sampling a clean sequence $x_0$, $t$ proportion of tokens are randomly selected to mask for computing an ELBO estimate. There are various advantages of these methods.
\vspace{-.5em}
\begin{itemize}[leftmargin=.9em]
    \item  First, ELBO estimation is computationally \textit{efficient}. The masking operation itself is extremely cheap, requiring only random token selection. The main computational cost comes from model inference on a fixed number (Monte-Carlo sample size) of masked sequences. In practice, however, RL remains effective even with a small sample size (e.g. $4$) \citep{yang2025mmada,ou2025principled}.
    \item Moreover, ELBO objective is well \textit{aligned with the pre-training}, since both objectives rely on randomly masked sequence (\Cref{sec:mdm}). This makes RL using a randomly masked sequence a principled and widely adopted choice.
\end{itemize}
Despite these strengths, RL with ELBO surrogates introduces non-negligible bias that can potentially cause training collapse, which we diagnose in the next section.



\subsection{Bias in ELBO-based Objectives}  

In RL based on sequence likelihood, ELBO surrogate and the special decoding process of dLLMs jointly introduce a mismatch between the policy optimized during training and the policy used at inference, termed \textit{Training-Inference Mismatch (TIM)} \citep{qi2025defeating}. First, the ELBO is a lower bound of the true likelihood with a non-negligible gap \citep{jiang2026diffusion} and estimation error, which leads to potentially large bias and variance in the likelihood ratio \citep{tang2025wd1,zhong2026stabilizing}, further degrade performance \citep{jiang2026diffusion} and even cause training collapse \citep{zhong2026stabilizing}. Second, prevailing dLLM decoders are iterative re-masking and block-wise, producing a sampling distribution differs from the training policy. 

Formally, RL methods such as policy gradient (PG) leverage the importance-sampling ratio ${\pi_\theta(x_0)}/{{\pi_{\text{old}}}(x_0)}$ for importance sampling following distribution ${\pi_{\text{old}}(x_0)}$. But current methods replace marginal likelihood with ELBO: $\textcolor{red!80}{\hat{\pi}_{\text{old}}(x_0) = {L}(x_0; p_\text{old})}$. The sampling process discretizes the continuous reverse process into $K$ steps $\{t_1, \cdots, t_K\} \in [0, 1]$, inducing a sampling distribution:
\begin{align}
\textcolor{blue!80}{\pi^\text{rm}_{\text{old}}(x_0)} = \prod_{k=1}^{T} \underbrace{m(\sigma_{k} \mid x_{t_k}; p_\text{old})}_{\text{remasked token selection}} \cdot \prod_{n \in \sigma^c_{k}} \underbrace{p_\text{old}(x_0^{(n)} \mid x_{t_k})}_{\text{token prediction}},
\end{align}
where $m$ is the selection policy to decide which tokens to remask, such as low-confidence selection, and $\sigma^c_k$ denotes the set of unmasked tokens, the complement of the predetermined remasking set $\sigma_k$. This leads to a clear mismatch $\textcolor{blue!80}{{\pi}^\text{rm}_\text{old}(x_0)} \neq \textcolor{red!80}{\hat{\pi}_\text{old}(x_0)}$, further causing the likelihood-ratio correction ineffective and objective bias in ELBO-based methods that are based on PG or PPO \citep{wang2025spg,ou2025principled}:
\begin{align}
\mathcal{L}_{\text{PG}}=\mathbb{E}_{ x_0 \sim \textcolor{blue!80}{\pi^\text{rm}_{\text{old}}}}
\big[
\frac{\hat{\pi}_\theta(x_0)}{\textcolor{red}{\hat{\pi}_\text{old}(x_0)}} A \big],\ \mathcal{L}_{\text{PPO}}=\mathbb{E}_{ x_0 \sim \textcolor{blue!80}{\pi^\text{rm}_{\text{old}}}}
\bigg[
\sum_{n=1}^N \min\big(\frac{\hat{\pi}_\theta(x_0^{(n)})}{\textcolor{red!80}{\hat{\pi}_\text{old}(x_0^{(n)})}} A, \text{clip}(\frac{\hat{\pi}_\theta(x_0^{(n)})}{\textcolor{red!80}{\hat{\pi}_\text{old}(x_0^{(n)})}}, 1\pm \epsilon) A \big)\bigg]. \nonumber
\end{align}
Such TIM bias, including even those arising from numerical rounding errors, has been shown to lead to catastrophic training collapse \citep{zhong2026stabilizing,qi2025defeating,liu-li-2025-rl-collapse}. To address this issue, we propose reducing the RL to an equivalent off-policy self-distillation problem without relying on likelihood.



%% file: neurips2026/method.tex
\section{Method}

\subsection{RL as Guided Denoising}
\label{sec:reformulation}
Prominent RL methods such as PPO and GRPO \citep{schulman2015trust,sutton1999policy,shao2024deepseekmath} restrict the policy update through a clipping operator on the importance ratio, which is well-suited to autoregressive policies whose likelihoods are tractable. For dLLMs, since likelihood is intractable, we instead adopt a reverse-KL penalty, which likewise preserves the monotonic improvement property \citep{tang2025wd1}:
\begin{align}
\max_{\pi} \; \mathbb{E}_{x_0 \sim \pi} [ \psi A(x_0) -  D_{{\text{KL}}}(\pi \| \pi_{\text{old}} ) - \beta D_{\text{KL}} ( \pi \,\|\, \pi_{{\text{ref}}} ) ],
\end{align}
where $\psi$ is the guidance coefficient, $\beta$ is the regularization coefficient, and $\pi_{\text{ref}}$ is the (frozen) reference model.
Using the method of Lagrange multipliers, the objective has a closed-form solution $\pi^*$, satisfying for any output clean sequence $x_0$:
\begin{align}
\pi^{*}(x_0) \propto \pi_{\text{old}}(x_0)^{{(1 - \beta)}} \cdot \pi_{{\text{ref}}}(x_0)^{{\beta}}  \cdot
\exp\big( \psi A(x_0) \big).
\label{eq:rev_grpo_update}
\end{align}
Here $\pi^*$ denotes the optimal policy of the reverse-KL regularized RL problem, a marginal distribution over clean sequences. In this work, we aim to learn a masked diffusion model (MDM) whose induced sampling distribution matches $\pi^*$. To achieve this, we first restrict the \textit{training-time} forward process: 
\begin{assumption}[Consistent Forward Process] We apply the same forward process masking on all the diffusion models (policies including reference model, old model, and current parametrized model). For simplicity, we let $q(x_t|x_0)$ denote the forward process. 
\label{def:idendp}
\end{assumption}
Assumption \ref{def:idendp} genuinely reflects the practice of likelihood approximation with randomly masked sequences, i.e., obtaining the same batch of masked sequences to compute the ELBO for all models ($p_\theta$, $p_\text{old}$, $p_\text{ref}$) \citep{yang2025mmada,ou2025principled}. Based on a consistent forward process, the RL policy update can then be converted into guided sampling: \looseness=-1
\begin{lemma}[\textbf{Energy-Guided Denoising Distribution} \citep{lu2023contrastive,tang2025wd1}]
Based on consistent masking in Assumption \ref{def:idendp}, the closed-form solution to reverse-KL regularized RL in \Cref{eq:rev_grpo_update} induces an energy-guided MDM $p^*$, which we denote as a \textbf{teacher} denoiser, satisfying $\forall x_0, t, x_t$: 
\begin{align}
p^*(x_0|x_t) = p_\text{old}^\text{ref}(x_0|x_t) \cdot \exp(\psi A(x_0) - A_t(x_t)),
\label{eq:inter_energy}
\end{align}
where $p_\text{old}^\text{ref}(x_0|x_t)\propto p_{\text{old}}(x_0|x_t)^{1-\beta} p_{\text{ref}}(x_0|x_t) ^{\beta}$, the negative energy guidance $A(\cdot)$ is the advantage function, and the log-normalization-constant $A_t(x_t) = \log \mathbb{E}_{x_0 \sim p_\text{old}^\text{ref}(\cdot|x_t)}[\exp \big(\psi A( x_0)\big)]$.
\label{lemma:ieg}
\end{lemma}

In particular, \Cref{eq:inter_energy} shows that the teacher denoising denoiser $p^*$ is a MDM with base model $p_{\text{old}}^{\text{ref}}$, guided by the energy function $\mathcal{E}(x_0)=-\psi A(x_0)$. Consequently, we can directly distill from the guided self-teacher $p^*$, and thus reduce the RL policy optimization problem to self-distillation to bypass the bias in ELBO-based methods.

\subsection{Guided Denoiser Self-Distillation}
\label{sec:gdsd}

\begin{figure}[t!]
    \centering
    \includegraphics[width=.99\linewidth]{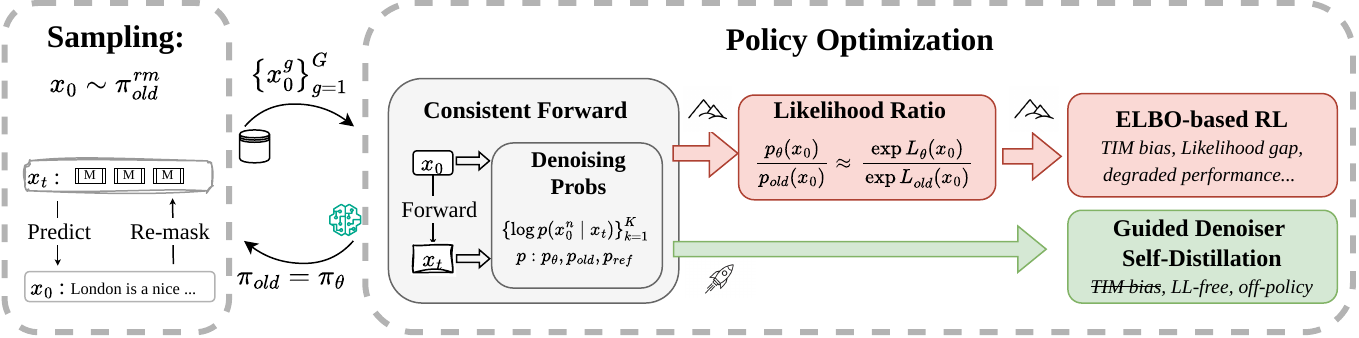}
    \caption{Overview of reinforcement learning for diffusion Large Language Models (dLLMs). Prior methods are based on ELBO $L_\theta(x_0)$ to approximate the likelihood ratio, which introduces training-inference-mismatch (TIM) bias. Our method \algfullname conducts denoiser self-distillation directly, which bypasses likelihood computation and TIM bias.
    }
    \label{fig:dllm_rft} 
\end{figure}

To directly approximate the guided denoising distribution $p^*$ for masked diffusion models (MDMs), we propose \textbf{\algfullname (\algname)}, a framework that trains $p_\theta$ by distilling from the teacher energy-guided denoiser $p^*$. Concretely, \algname minimizes the distance between the parametrized denoiser $p_\theta(x_0 | x_t)$ and the teacher distribution $p^*(x_0 | x_t)$ via logit matching \citep{hinton2015distilling}: \looseness=-1
\begin{myframe}[title=\algfullname (\algname)]
\begin{align}
\mathbb{E}_{t \sim U(0,1),\ x_0 \sim \pi_{\text{old}},\ x_t \sim q(\cdot|x_0)}\Big[ \log p_\theta(x_0 |x_t) -    \underbrace{ \big( \log p_\text{old}^\text{ref}(x_0|x_t) + \psi A(x_0) -  A_t(x_t) \big) }_{\log p^*(x_0|x_t)}  \Big]^2.
\label{eq:egpo}
\end{align}
\end{myframe}
The sequence-level denoising distribution factorizes into token-level denoising models \cite{shi2025simplifiedgeneralizedmaskeddiffusion,zhao2025d1,yang2025mmada,tang2025wd1}: $\log p_\theta(x_0| x_t) = \sum_{n=1}^N \log p_\theta(x_0^{(n)}| x_t)$, where $p_\theta(x_0^{(n)} | x_t)$ is the token-level denoiser of the current dLLMs. The advantage function $A$ is estimated from rewards. The main challenge in directly optimizing \Cref{eq:egpo} is the intractable normalization constant $A_t$, which we eliminate in \Cref{sec:practical_design} with normalization-free reformulation. In the rest of this section, we highlight a few advantages of \algname.

\textbf{RL via Denoiser Self-Distillation.} \algname retains the main advantages of ELBO-based methods that leverage randomly masked sequences: computational efficiency and natural alignment with the pre-training objective. Unlike prior approaches, however, \algname eliminates the need for ELBO due to its likelihood-free formulation. The objective leverages off-policy samples, sampled through iterative self-refinement, and distills the dLLM denoiser from a teacher, obtained from dLLM's old logits shifted by the advantage. Such guided self-distillation can increase the denoising probability of clean samples with a positive advantage, and decrease it for those with a negative advantage (i.e., negative samples). As a result, \algname naturally handles negative samples without requiring additional designs, in contrast to prior works \citep{tang2025wd1,wang2025spg}.

\textbf{Bypassing TIM Bias.}  \algname does not form a likelihood ratio in its objective, so the two structural biases of importance-sampling (\Cref{sec:tim}), including the ELBO–likelihood gap and the re-masking sampler–policy gap, never enter the loss objective by design. Distillation originally allows for training on off-policy samples, which absorbs the potential harms by the special process of dLLMs decoding.


\textbf{Minor Changes to RL Pipeline.} \algname can be implemented with only minor modifications to existing RL pipelines. We summarize the standard RL pipeline for dLLMs in \Cref{fig:dllm_rft}.  In the sampling stage, we follow existing dLLM RL workflows (e.g., Diffu-GRPO \citep{zhao2025d1}), using $\pi_{\text{old}}$ to generate a group of sequences through iterative prediction and re-masking. We then follow the standard sequence-likelihood computing procedure: we randomly sample several time steps $t$ and apply a consistent forward process by drawing $x_t \sim q(\cdot | x_0)$. At each sampled step, we compute the denoising log-probabilities $\{\log p(x_0^k | x_t)\}_{k=1}^K$ under all relevant policies, namely $\pi_{\text{old}}$, $\pi_{\text{ref}}$, and $\pi_\theta$. ELBO-based methods use these denoising probabilities to estimate the ELBO and then substitute it for the likelihood in the RL objective. By contrast, \algname is likelihood-free and computes its loss directly from the collected denoising probabilities. \looseness=-1

\subsection{Normalization-Free Optimization} 
\label{sec:norm_free}

The main challenge in implementing \algname in \Cref{eq:egpo} is computing the normalization constant in the teacher denoiser $p^*$. We address this by firstly decomposing the geometric mixture model $p_{\text{old}}^{\text{ref}}$ into $p_{\text{old}}$ and $p_{\text{ref}}$ , and gather the normalization term of mixture policy $p_{\text{old}}^{\text{ref}}$ and $A_t$ into a single constant $Z_t$: 
\begin{align}
\log {p}^*(x_0|x_t) = (1-\beta) \log p_\text{old}(x_0|x_t) + \beta \log p_\text{ref}(x_0|x_t) + \psi A(x_0) - \colorbox{red!30}{$\log Z_t(x_t)$},
\label{eq:egdd_upper}
\end{align}
where $Z_t$ is the partition function, which requires a summation over the exponential completion space $\mathcal{X}$:
$$ Z_t(x_t) = \sum_{x_0 \in \mathcal{X}} p_{\text{old}}(x_0|x_t)^{(1-\beta)} p_{\text{ref}}(x_0|x_t)^{\beta} \exp \big( \psi A(x_0) \big).$$
Estimating $Z_t$ requires drawing additional samples of $x_0$ (e.g., sampling from $p_\text{old}(x_0|x_t)$) and performing model inference on them, resulting in extra computational overhead. To make \algname efficient for online RL, we therefore propose two normalization-free logit-matching methods to bypass the bottleneck of computing $Z_t$ by exploiting the translation invariance of the Softmax operator.



\textbf{Naive Solution: Direct Matching.} 
A naive method is to directly match the student logits to the unnormalized teacher denoiser in \Cref{eq:egdd_upper} without the
partition term $\log Z_t$. This is motivated by the translation invariance of the Softmax operator:
$\operatorname{Softmax}(\mathbf{y})=\operatorname{Softmax}(\mathbf{y}+c)$ for any constant $c$. Therefore, given masked sequence $x_t$, if the logits of the student denoiser match the teacher's up to an additive term independent of $x_0$, the resulting denoising distribution is identical to $p^*$ after Softmax.


\textbf{Token-level Logit Centralization (TLC).} Direct matching is simple and effective. However, during iterative matching to the unnormalized target, the scale of the logits can shift uncontrollably. Inspired by zero-meaned logit matching by \citet{hinton2015distilling}, we propose an alternative method: centralizing the logits of denoisers. For any MDM $p$, we define the token-level logit-centralized model $\bar{p}$ as:
\begin{align}
\log \bar{p}(x_0|x_t) \overset{\text{def}}{=}  \sum_n^{N} \log p(x^{(n)}_0|x_t) - \sum_n^{N} \frac{1}{|\mathcal{V}|} \sum_{x'^{(n)}_0 \in \mathcal{V}} \log p(x'^{(n)}_{0}|x_t)
\label{eq:def_lc}
\end{align} 
Due to the factorization of the sequence-level denoising distribution, TLC is equivalent to centralizing the sequence-level log-probability: $\log \bar{p}(x_0|x_t) = \log p(x_0|x_t) - \sum_{x_0} \log p(x_0|x_t)$. By the translation invariance of the softmax operator, the denoising probability distribution induced by these centralized logits is identical to the one induced by the raw logits. Since the partition function $Z_t$ is independent of $x_0$, centralization eliminates $Z_t$ entirely from the objective:
\begin{proposition} 
Define the advantage baseline $b:=\sum_{x_0 \in \mathcal{X}} A(x_0) / |\mathcal{X}|$. For any clean and masked sequence $x_0$ and $x_t$, the centralized logits of $p^*$ becomes normalization-free (Proof in \ref{sec:proof_centering}):
\begin{align}
\log \bar{p}^*(x_0|x_t) = (1-\beta)\log \bar{p}_\text{old} (x_0|x_t) + \beta \log \bar{p}_\text{ref} (x_0|x_t) + \psi {A}(x_0) - b.
\label{eq:centering}
\end{align}
\label{proposition:centering}
\end{proposition} 
\vspace{-1.3em}
Since the advantage baseline $b$ is independent of $x_0$, diffusion time $t$, and even masked sequence $x_t$, being a globally constant offset to all the logits at all time $t$, in practice we can further omit the baseline $b$ directly, meanwhile preserving the resulting denoising distribution. Additionally, since the logits are centralized and the advantage itself has zero mean, i.e., $\mathbb{E}_{\pi_\text{old}}[A(x_0)]=0$, the scale of the logits is prevented from drifting away during iterative updates. Combined with normalization-free implementation, we derive the practical objective of GDSD (additional details in \ref{sec:practical_design}):
\begin{myframe}[title=\algfullname (\algname) (practical)]
    \begin{align}
\mathbb{E}_{{t \sim U(0,1),\ x_0 \sim \pi^{\text{rm}}_{\text{old}},\ x_t \sim q(\cdot|x_0)}}\Bigg[ \Big( \log \frac{\bar{p}_\theta(x_0|x_t)}{\bar{p}_{\text{old}}(x_0|x_t)} - \psi {A}(x_0) \Big)^2 +  \beta \Big( \log \frac{\bar{p}_\theta(x_0|x_t)}{\bar{p}_\text{ref}(x_0|x_t)}\Big)^2 \Bigg].
\label{eq:egdd_practical}
\end{align}
\end{myframe}

\section{Connection and Comparison to Related Work}
\label{sec:connection}

In this section, we draw the important connections between our method and closely related work.

\algname applies logit matching to approximate the teacher denoising distribution $p^*$ with $p_\theta$. 
Specifically, the objective is the squared $l_2$ distance between the logit vectors of the student denoiser and the teacher with samples from a mixture of the old and reference policy. 
Apart from $l^2$, it is also eligible to apply other functions for self-distillation, such as forward-KL and reverse-KL divergence. 

\textbf{Advantage-Weighted ELBO.} Adopting forward-KL divergence for distillation derives advantage-weighted ELBO objectives. This formulation recovers the objectives employed in wd1 (AW-DCE) \citep{tang2025wd1} and DMPO \citep{zhu2025enhancing}. Define the forward-KL objective as:
$$\mathcal{L}_{\text{fwd}}(\theta) \overset{\text{def}}{=} \mathbb{E}_{t \sim U(0,1),\ x_t \sim p^*_t} [D_{\text{KL}}(p^*(\cdot|x_t) \parallel p_\theta(\cdot|x_t))]. $$
Since the partition function $A_t(x_t)$ and $p_{\text{old}}^{\text{ref}}$ terms in $p^*$ are independent of $\theta$, then importance sampling reveals this is equivalent to the ELBO re-weighted by the exponential of the advantage:
\begin{proposition}[Forward-KL Distillation Induces Advantage-Weighted ELBO]
\label{prop:forward_kl_aw_elbo}
Up to terms independent of $\theta$, optimizing $\mathcal L_{\rm fwd}$ is equivalent to an ELBO-style denoising loss reweighted by the exponential advantage (proof in \ref{proof:forward_kl_aw_elbo}):
\begin{equation}
\nabla_\theta \mathcal L_{\text{fwd}}(\theta)
\propto
\nabla_\theta
\mathbb E_{x_0 \sim p^{\text{ref}}_{\text{old}}}
\Big[
\exp(\psi A(x_0))
\underbrace{\mathbb E_{t \sim U(0,1),\, x_t \sim q(\cdot \mid x_0)}
\big[
-\log p_\theta(x_0 \mid x_t)
\big]}_{\text{ELBO}}
\Big].
\nonumber
\end{equation}
\vspace{-1.5em}
\end{proposition}
A primary limitation of such exponential-weighting schemes is data inefficiency \citep{park2025generalizing}: samples $x_0$ with negative advantages are assigned negligible weights, contributing to the gradient mildly. This makes the exponentially weighted method weaker than policy gradient methods. While \citet{tang2025wd1} attempts to mitigate this by introducing an auxiliary loss to penalize negative samples, such explicit penalization can trigger training instability \citep{ou2025principled}. In contrast, \algname naturally incorporates the sign of the advantage function, allowing for the stable and efficient utilization of the entire sample batch.

\textbf{Regularized Policy Gradient.} Applying reverse-KL distillation to the energy-guided teacher recovers an policy-gradient-based regularized RL objective (proof in \ref{proof:rkl_pg}):
\begin{align}
&\mathcal{L}_{\text{rev}}(\theta) \overset{\text{def}}{=} \mathbb{E}_{t \sim U(0,1), x_t \sim p_{\theta,t}} \big[D_{\text{KL}}( p_\theta(\cdot|x_t)) \parallel p^*(\cdot|x_t)\big] \nonumber 
\\ = &\underbrace{\mathbb{E}_{x_0 \sim \pi_\theta} [ - \psi 
A(x_0)]}_{\text{reward maximization}} + \underbrace{\mathbb{E}_{t \sim U(0,1), x_0 \sim \pi_\theta ,x_t \sim q(\cdot|x_0)} [A_t(x_t)]}_{\text{reward baseline: diffusion state value}} + \underbrace{ \mathbb{E}_{x_0 \sim \pi_\theta} [ L(x_0; p_\theta)- L(x_0; p_\text{old}^\text{ref}) ]}_{\text{regularization}},\nonumber 
\end{align}
where $L$ denotes ELBO. This decomposition shows that reverse-KL distillation recovers a regularized
policy gradient (PG) form. The leading term corresponds to maximizing the expected
advantage under the current policy $\pi_\theta$, while $A_t(x_t)$ acts as a
state-dependent baseline. The
remaining term regularizes the current denoiser against the old/reference
mixture $p_\text{old}^\text{ref}$. Therefore, PG-based methods can
be viewed as optimizing a related objective, but with the likelihood term approximated by ELBO, including SPG \citep{wang2025spg}, UniGRPO \citep{yang2025mmada}, ESPO \citep{ou2025principled}. However, as detailed in \Cref{sec:tim}, all of these methods that require importance sampling, such as PG for dLLMs, have inevitable TIM bias.

\textbf{Reinforcement Learning via Logit-Matching.} Logit-matching loss has been used in \textit{auto-regressive} models post-training \citep{zhu2023fine,wu2024self,tang2025rspo}, including the large-scale one in Kimi \citep{team2025kimi}.
Concurrent work EMBR \citep{shankar2026energy} applies similar logit-matching for dLLMs in preference learning, which is derived from an offline RL objective and restricted to preference data. Recently, logit-matching loss has also been shown to be exceptionally stable in continuous diffusion RL \citep{choi2026rethinking}. In contrast, \algname is a normalization-free method in general RL settings for (discrete) dLLMs. (More related work in \ref{append:related_work})

%% file: neurips2026/experiments.tex
\section{Experiment}
\label{sec:exp}

We perform RL based on two open-source diffusion LLMs, LLaDA-8B-Instruct and Dream-v0-Instruct-7B, evaluated across six benchmarks spanning three domains: mathematical reasoning (GSM8K, MATH500), planning (Countdown, Sudoku), and coding (HumanEval, MBPP). All trainings except coding tasks are based on Low-Rank Adaptation (LoRA). A separate model is trained for each task, except for coding, where a single model is trained on AceCoder-87K and evaluated on both HumanEval and MBPP. Training for planning and coding tasks follows the reward configuration (verifiable reward) following ESPO \citep{ou2025principled}, while on the mathematical reasoning we incorporate format-reward following d1 \citep{zhao2025d1}. (See more details in \ref{append:exp_setup}) 

\textbf{Evaluation.} Given the incompatible evaluation protocols across prior literature, we follow the ESPO settings and directly adopt their reported results for LLaDA, d1, wd1, and UniGRPO (detailed in Appendix~\ref{app:exp_hyper}); math results are reproduced from the official repositories and evaluated via lm-eval with diffusion steps equal to the generation length. To unify evaluation, we use zero-shot for Sudoku, Countdown, GSM8K, MATH500, and HumanEval, and 3-shot for MBPP, with logical planning tested at generation lengths of 128, 256, and 512. 

\textbf{Implementation.} We implement mainly two versions of \algname, both of which remove the intractable normalization constant (i.e. $Z_t$ in \Cref{eq:egdd_upper}). One is default \algname direct matching without token logit centralization (TLC) and the normalization constant is removed (i.e. \Cref{eq:egdd_practical} but without TLC), which we name \algname directly. We also implement \algname with TLC following the same equation. For baseline implementation and their results, we either reproduce results by evaluating the released checkpoints or re-run training using their official implementation.

\subsection{Main Results}

We evaluate \algname on the pure diffusion model Dream-7B in Table \ref{exp:dream}. Compared to prior state-of-the-art ELBO-based methods, \algname significantly improves the average test accuracy $+9.5\%$ on average and $+10\%$ with the best-performing generation length. Adding token logit centralization yields an additional gain of roughly, achieving up to $+19.6\%$ over the baselines. These results confirm the effectiveness of both \algname and the centralization. 

\begin{table}
\caption{Zero-shot test accuracy of different methods using diffusion language model Dream-7B-Instruct. We summarize the best performance across generation length in the right figure. \algname demonstrates significant improvements over ELBO-based methods, achieving up to $+19.6\%$.}
\label{exp:dream}
\vspace{.3em}
\centering
\scriptsize
\begin{tabular}{l ccc>{\columncolor{gray!20}}c ccc>{\columncolor{gray!20}}c}
\toprule
& \multicolumn{4}{c}{\textbf{Sudoku}} & \multicolumn{4}{c}{\textbf{Countdown}} \\
\cmidrule(lr){2-5} \cmidrule(lr){6-9}
\textbf{Model / Seq Len}
& \textbf{128} & \textbf{256} & \textbf{512} & \textbf{Avg.}
& \textbf{128} & \textbf{256} & \textbf{512} & \textbf{Avg.} \\
\midrule

\textbf{Dream-7B}
& 9.3 & 2.1 & 14.0 & 8.5
& 8.5 & 7.8 & 17.4 & 11.2 \\

+ diffu-GRPO (d1)
& 64.4 & 69.7 & 51.1 & 61.7
& 27.3 & 27.7 & 37.5 & 30.8 \\

+ wd1
& 29.5 & 39.2 & 30.3 & 33.0
& 28.9 & 37.9 & 42.2 & 36.3 \\

+ ESPO
& 71.7 & 72.3 & 71.3 & 71.8
& 68.8 & 66.8 & 64.8 & 66.8 \\

\midrule
+ \textbf{\algname} (ours)
& 82.3 & 82.2 & 79.3 & 81.3
& 70.3 & 68.4 & 63.7 & 67.5 \\


+ \textbf{\algname w/ TLC} (ours) &\textbf{91.1}&\textbf{91.3}&\textbf{91.7}& \textbf{91.4}&\textbf{83.2}&\textbf{82.4}&\textbf{84.8}
& \textbf{83.5} \\
\bottomrule
\end{tabular}
\begin{subfigure}{0.3\textwidth}
\includegraphics[width=\linewidth,valign=m]{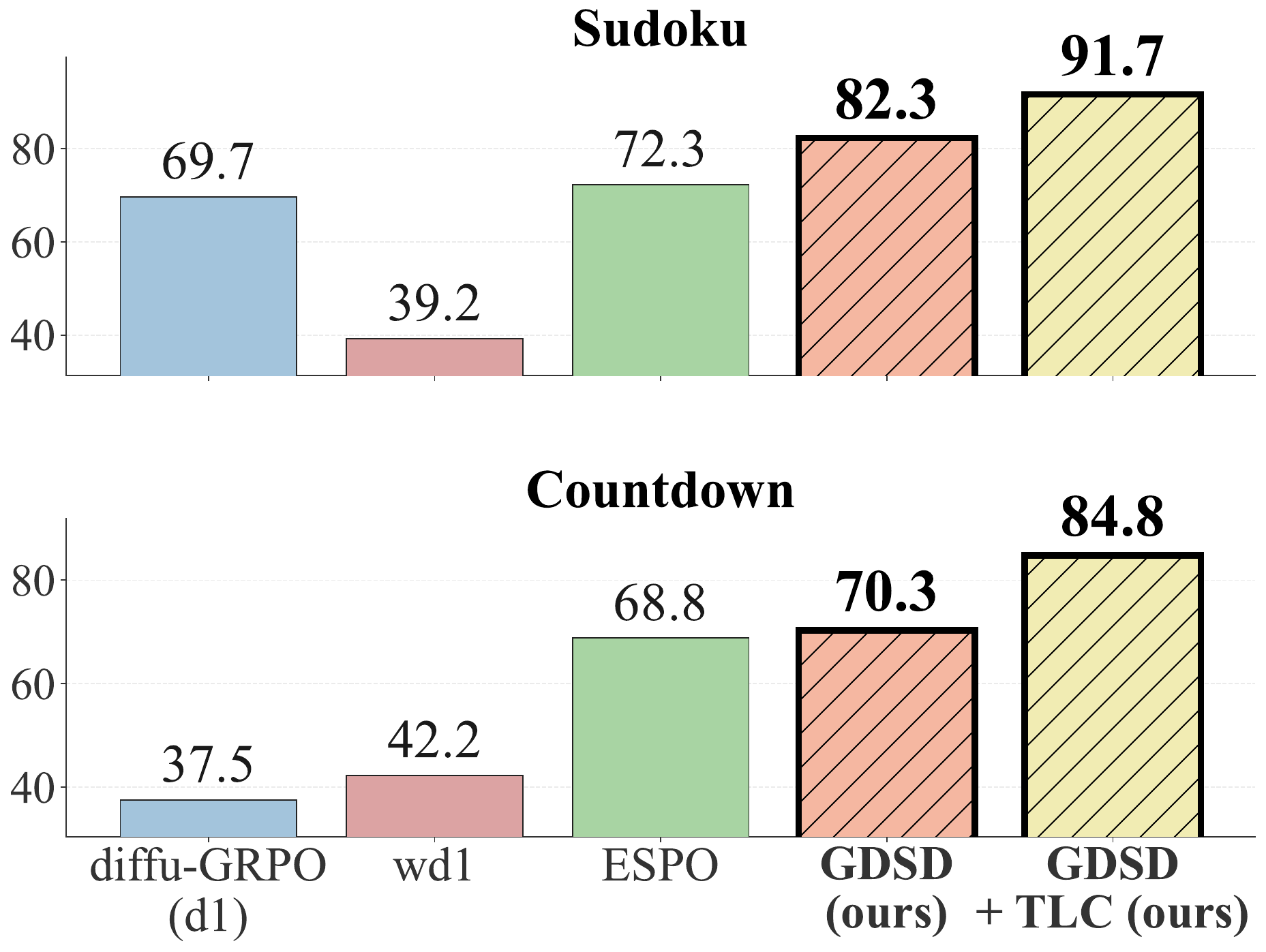}
\end{subfigure}
\vspace{-1em}
\end{table}

For the block diffusion model \citep{arriola2025block}, the most widely tested base model LLaDA-8B, we conduct a comprehensive evaluation on planning, math, and coding benchmarks. We provide the training reward dynamics in \Cref{fig:reward_curves}, and testing accuracy in  \Cref{tab:llada_all} and \Cref{tab:llada_code}. As for training reward, on the planning and coding tasks, GDSD demonstrates stable convergence relative to prior state-of-the-art ELBO-based methods. As for testing accuracy, \algname outperforms prior methods on almost all the benchmarks, achieving average accuracy gains up to $+5\%$ on Sudoku, $+1.9\%$ on Countdown, $+3\%$ on GSM8K and $+1.1\%$ on MATH, $+4.6\%$ on HumanEval-Plus.

These results, particularly the accuracy improvement of up to $+20\%$, demonstrate the effectiveness of our denoiser self-distillation method reformulated from RL. Even when GDSD performs comparably to ELBO-based baselines on some benchmarks, the results still support our central claim: ELBO is not necessary for RL with dLLMs and can be removed to avoid the bias analyzed in \Cref{sec:tim}. Overall, the empirical results validate GDSD as an effective alternative to ELBO-based importance-sampling RL, providing a cleaner and more principled optimization route.
\vspace{-.5em}

\subsection{Ablation Study}
In this section, we perform ablation study on different elements of \algname to confirm the effectiveness of our design choices.

\textbf{Token Logit Centralization (TLC).} Across all benchmarks in \Cref{tab:llada_all}, \algname by direct matching without TLC achieves clear improvements over prior ELBO-based methods. This confirms the effect of the normalization constant is negligible in practice. Notably, although \algname with TLC is more faithful to the theory, it sometimes leads to degraded test performance. However, the lack of consistent improvement in test accuracy suggests a potential generalization gap. We hypothesize that, by self-centralizing, TLC makes the update focus more strongly on relative logit differences, which improves fitting to the training reward but may also amplify overfitting to
training-specific signals, which explains the stable training dynamics of applying TLC in \Cref{fig:reward_curves} and \Cref{fig:aba_study}.

\textbf{Guidance Coefficient $\psi$.} We further conduct an ablation study on the guidance strength $\psi$ in
\Cref{fig:aba_study} to examine the role of energy guidance in GDSD. The results
show that increasing $\psi$ generally leads to higher training rewards, indicating
that stronger energy guidance produces a teacher denoiser more biased toward
high-advantage samples. This trend supports the effectiveness of the proposed
guided-denoiser distillation formulation: GDSD directly distills the
energy-guided target distribution, and therefore can translate stronger guidance
signals into improved optimization performance.

\begin{figure}[t]
    \centering
    \begin{subfigure}{0.32\textwidth}
        \centering
        \includegraphics[width=\linewidth]{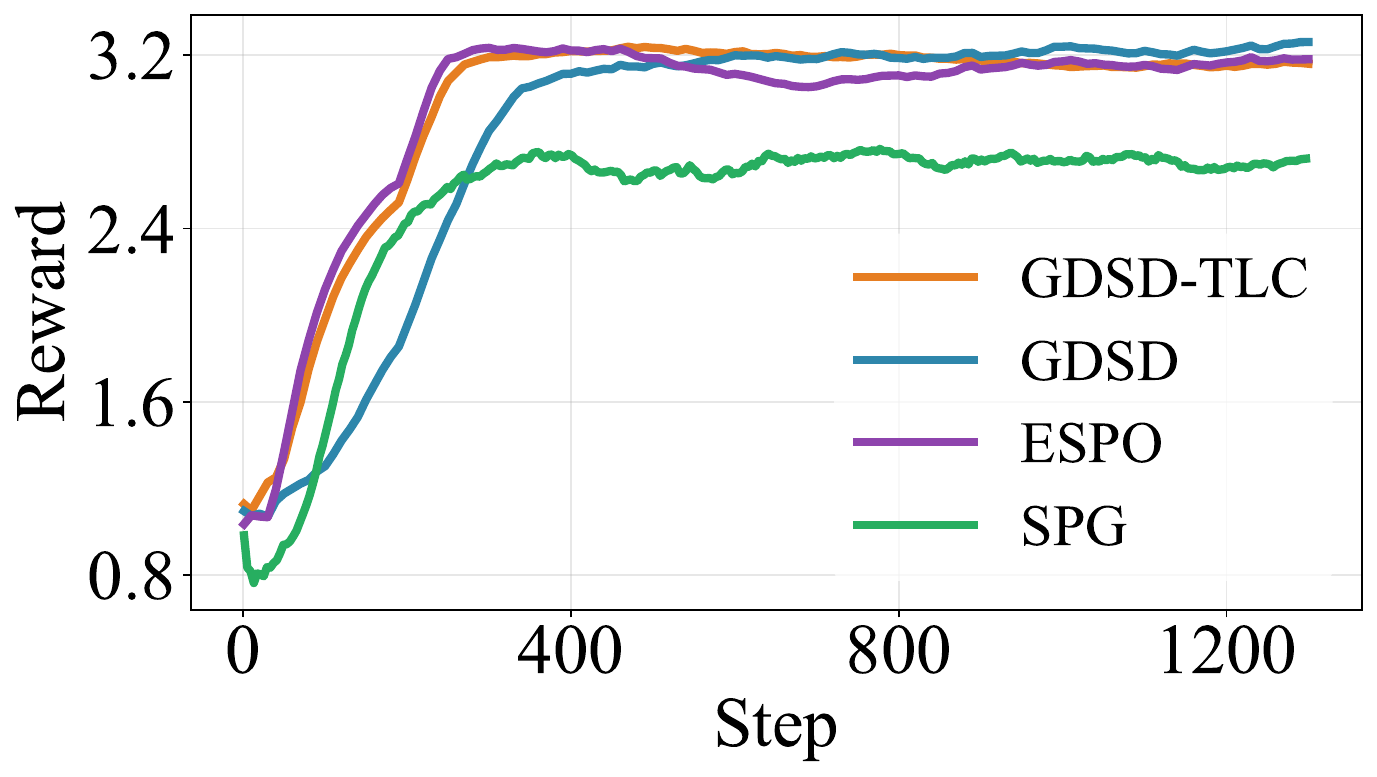}
        \label{fig:reward_gsm8k}
    \end{subfigure}
    \hfill
    \begin{subfigure}{0.32\textwidth}
        \centering
        \includegraphics[width=\linewidth]{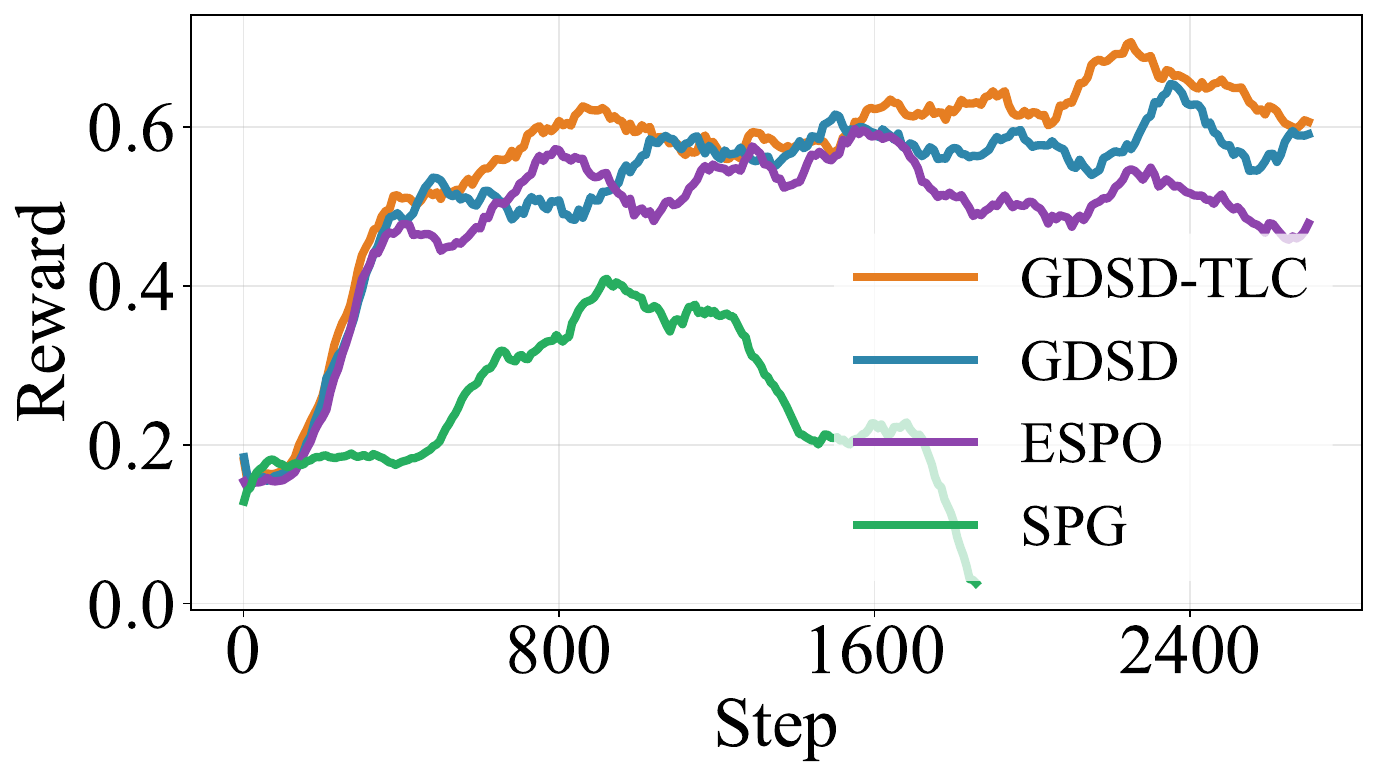}
        \label{fig:reward_countdown}
    \end{subfigure}
    \hfill
    \begin{subfigure}{0.32\textwidth}
        \centering
        \includegraphics[width=\linewidth]{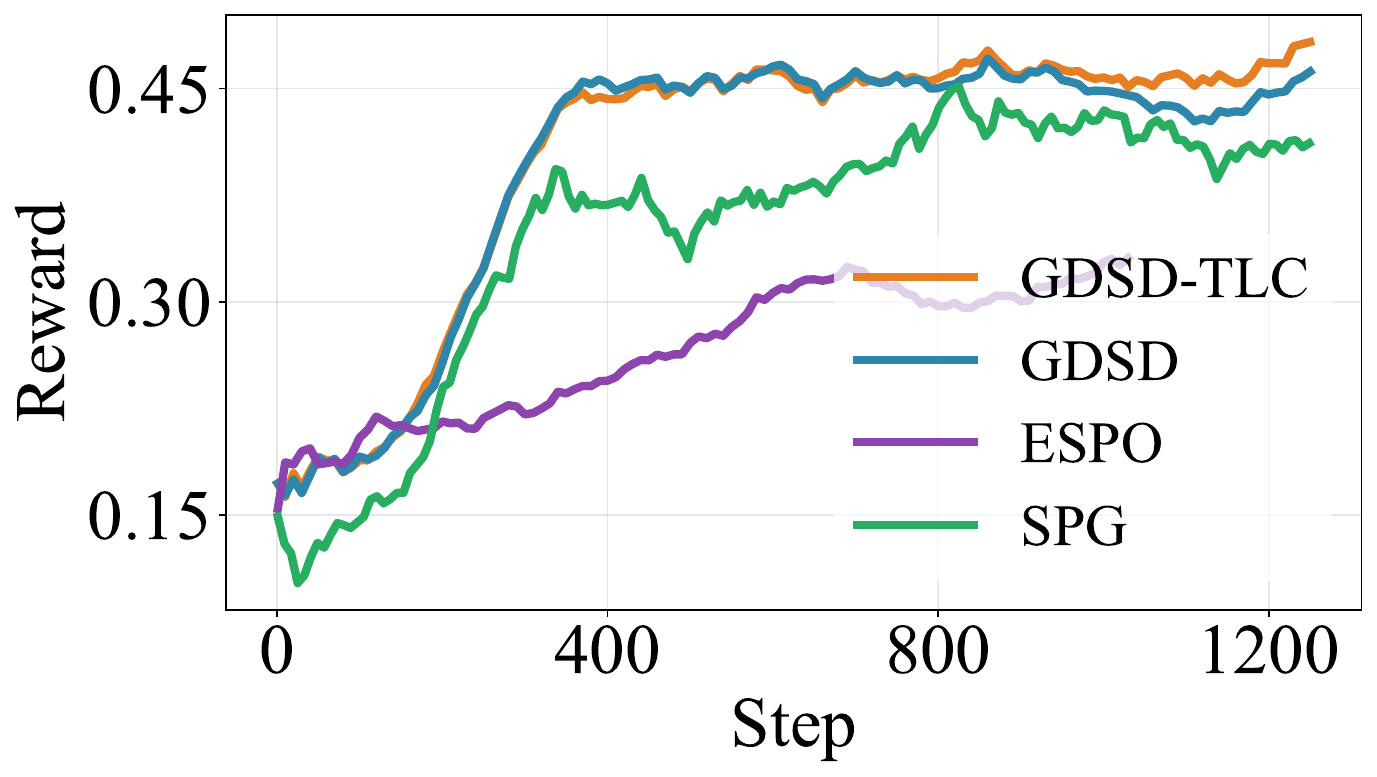}
        \label{fig:reward_code}
    \end{subfigure}
    \vspace{-1em}
    \caption{Training reward dynamics trained on LLaDA-8B-Instruct with different methods on different training datasets: \textbf{Left:} GSM8k; \textbf{Mid:} Countdown; \textbf{Right:} Coding. Our method demonstrates more stable training reward dynamics across datasets.}
    \label{fig:reward_curves}
\end{figure}

\begin{table*}[t] 
\centering 
\caption{Test accuracy of different methods using \textit{block} diffusion \citep{arriola2025block} model LLaDA-8B-Instruct. The results with $^{\dagger}$ are reproduced by us; others are extracted from the corresponding papers. Numbers following benchmark $(k)$ represent $k$-shot evaluation. Our methods demonstrate consistent improvement on average accuracy.} 
\label{tab:llada_all} 
\scriptsize
\begin{tabular}{l ccc>{\columncolor{gray!20}}c ccc>{\columncolor{gray!20}}c ccc>{\columncolor{gray!20}}c}
\toprule
& \multicolumn{4}{c}{\textbf{Sudoku(0)}}
& \multicolumn{4}{c}{\textbf{Countdown}}
& \multicolumn{4}{c}{\textbf{GSM8K}} \\
\cmidrule(lr){2-5} \cmidrule(lr){6-9} \cmidrule(lr){10-13}
\textbf{Model / Seq Len}
& \textbf{128} & \textbf{256} & \textbf{512} & \textbf{Avg.}
& \textbf{128} & \textbf{256} & \textbf{512} & \textbf{Avg.}
& \textbf{128} & \textbf{256} & \textbf{512} & \textbf{Avg.} \\
\midrule

\textbf{LLaDA-8B}
& 24.8 & 16.2 & 6.0 & 15.7
& 20.7 & 19.5 & 16.0 & 18.7
& 71.3 & 76.2 & 80.2 & 75.9 \\

+ diffu-GRPO (d1)
& 26.7 & 24.1 & 15.9 & 22.2
& 33.2 & 31.3 & 37.1 & 33.9
& 74.6 & 78.1 & 81.2 & 78.0 \\

+ wd1
& 22.6 & 76.4 & 62.8 & 53.9
& 47.7 & 51.2 & 46.1 & 48.3
& 77.2 & 80.8 & 82.3 & 80.1 \\

+ UniGRPO
& / & / & / & /
& 44.5 & 43.0 & 57.0 & 48.2
& 74.9 & 82.5 & 82.7 & 80.0 \\

+ DMPO
& 32.8 & 24.6 & 20.0 & 25.6
& 67.2 & 80.9 & \textbf{82.8} & 77.0
& 74.8 & 82.4 & 85.2 & 80.8 \\

+ SPG
& 87.3$^\dagger$ & 88.2$^\dagger$ & 71.7$^\dagger$ & 82.4
& 68.8 & 70.7 & 70.3 & 69.9
& 82.8$^\dagger$ & 85.1$^\dagger$ & 84.9$^\dagger$ & 84.2 \\

+ ESPO
& \textbf{92.7} & 84.7 & 80.5 & 86.0
& 81.6 & 82.0 & 79.3 & 81.0
& 81.6$^\dagger$ & 82.5$^\dagger$ & 83.0$^\dagger$ & 82.4 \\

\midrule

+ \textbf{\algname} (ours)
& 89.5 & 89.8 & 88.8 & 89.4
& 82.4 & \textbf{84.0} & \textbf{82.8} & \textbf{83.1}
& \textbf{83.9} & \textbf{86.4} & \textbf{86.0} & \textbf{85.4} \\

+ \textbf{\algname w/ TLC} (ours)
& 90.4 & \textbf{92.0} & \textbf{90.6} & \textbf{91.0}
& \textbf{85.6} & 80.5 & 75.0 & 80.4
& 81.3 & 85.6 & 85.0 & 84.0 \\

\bottomrule
\end{tabular}














\vspace{0.8em}

\scriptsize
\begin{tabular}{l ccc>{\columncolor{gray!20}}c ccc>{\columncolor{gray!20}}c ccc>{\columncolor{gray!20}}c}
\toprule
& \multicolumn{4}{c}{\textbf{MATH500}}
& \multicolumn{4}{c}{\textbf{HumanEval-Plus(0)}}
& \multicolumn{4}{c}{\textbf{MBPP(3)}} \\
\cmidrule(lr){2-5} \cmidrule(lr){6-9} \cmidrule(lr){10-13}
\textbf{Model / Seq Len}
& \textbf{128} & \textbf{256} & \textbf{512} & \textbf{Avg.}
& \textbf{128} & \textbf{256} & \textbf{512} & \textbf{Avg.}
& \textbf{128} & \textbf{256} & \textbf{512} & \textbf{Avg.} \\
\midrule

\textbf{LLaDA-8B}
& 34.4 & 35.2 & 41.4 & 37.0
& 23.2 & 30.5 & 41.5 & 31.7
& 36.2 & 42.0 & 38.1 & 38.8 \\

+ diffu-GRPO (d1)
& 34.9 & 36.6 & 41.7 & 37.7
& 22.0 & 29.9 & 37.2 & 29.7
& 34.8 & 36.6 & 38.0 & 36.5 \\

+ wd1
& 33.3 & 37.7 & 39.8 & 36.9
& 29.9 & 29.9 & 32.9 & 30.9
& 38.0 & 37.2 & 34.4 & 36.5 \\

+ UniGRPO
& 32.4 & 37.4 & 39.4 & 36.4
& / & / & / & /
& / & / & / & / \\

+ DMPO
& 30.0 & 38.2 & 42.8 & 37.0
& / & / & / & /
& / & / & / & / \\

+ SPG
& 33.4 & 40.0 & 41.8 & 38.4
& 32.9$^\dagger$ & 34.2$^\dagger$ & 37.8$^\dagger$ & 35.0
& 40.4$^\dagger$ & 40.8$^\dagger$ & 40.4$^\dagger$ & 40.5 \\

+ ESPO
& 36.0 & 39.0 & 43.4 & 39.5
& 24.4 & 36.6 & \textbf{42.7} & 34.6
& \textbf{43.6}$^\dagger$ & 43.2$^\dagger$ & 41.2$^\dagger$ & 42.7 \\

\midrule

+ \textbf{\algname} (ours)
& \textbf{37.0} & 40.4 & \textbf{44.6} & \textbf{40.6}
& 36.0 & 38.4 & 41.5 & 38.6
& 40.6 & 41.8 & \textbf{43.6} & 42.0 \\

+ \textbf{\algname w/ TLC} (ours)
& 34.6 & \textbf{41.6} & 44.0 & 40.1
& \textbf{38.4} & \textbf{39.6} & 39.6 & \textbf{39.2}
& 43.0 & \textbf{43.6} & 43.2 & \textbf{43.3} \\

\bottomrule
\end{tabular}
\vspace{-1em}
\end{table*}

\input{neurips2026/contents/temp_ablation}

\section{Conclusion}
In this work, we argue that existing RL methods for dLLMs that use the likelihood evidence lower bound (ELBO) as a likelihood surrogate introduce bias, which can degrade performance and potentially cause training collapse. We propose \algfullname{} (\algname), an off-policy self-distillation framework that equivalently performs reverse-KL-regularized RL without requiring likelihood ratios, thereby bypassing this bias. Empirically, \algname{} consistently improves over state-of-the-art ELBO-based RL methods on Dream-7B and LLaDA-8B across planning, math, and coding benchmarks. The gains are especially pronounced on Dream-7B, reaching up to $+20\%$ absolute improvement. These results demonstrate that \algname{} provides an effective and stable alternative for RL with dLLMs.

\input{neurips2026/contents/math_planning_results}

%% file: neurips2026/contents/temp_ablation.tex
\begin{figure}[t!]
    \centering
    \begin{subfigure}{0.32\textwidth}
        \centering
        \includegraphics[width=\linewidth,valign=t]{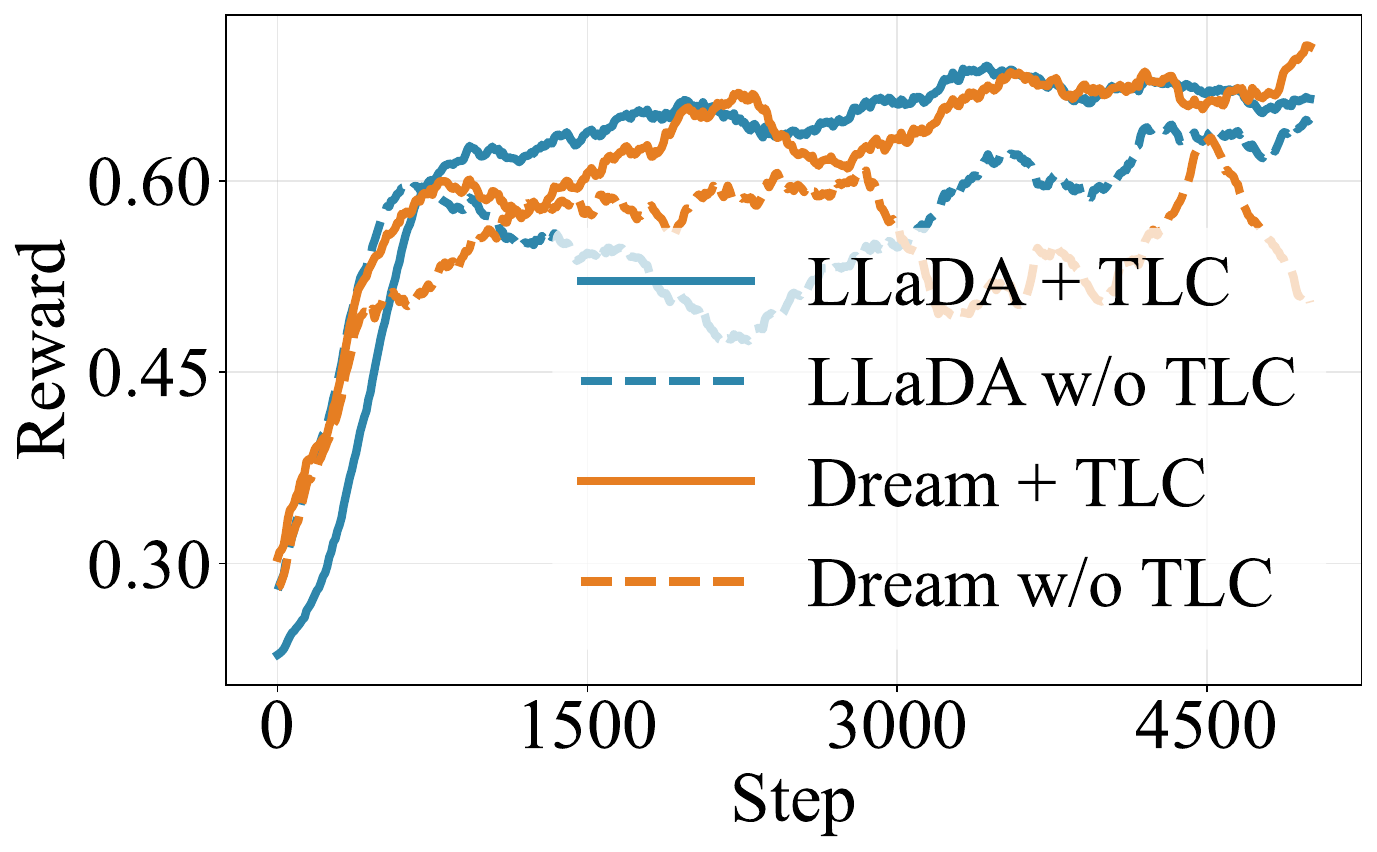}
        \label{fig:tlc_reward_ctd}
    \end{subfigure}
    \hfill
    \begin{subfigure}{0.32\textwidth}
        \centering
        \includegraphics[width=\linewidth]{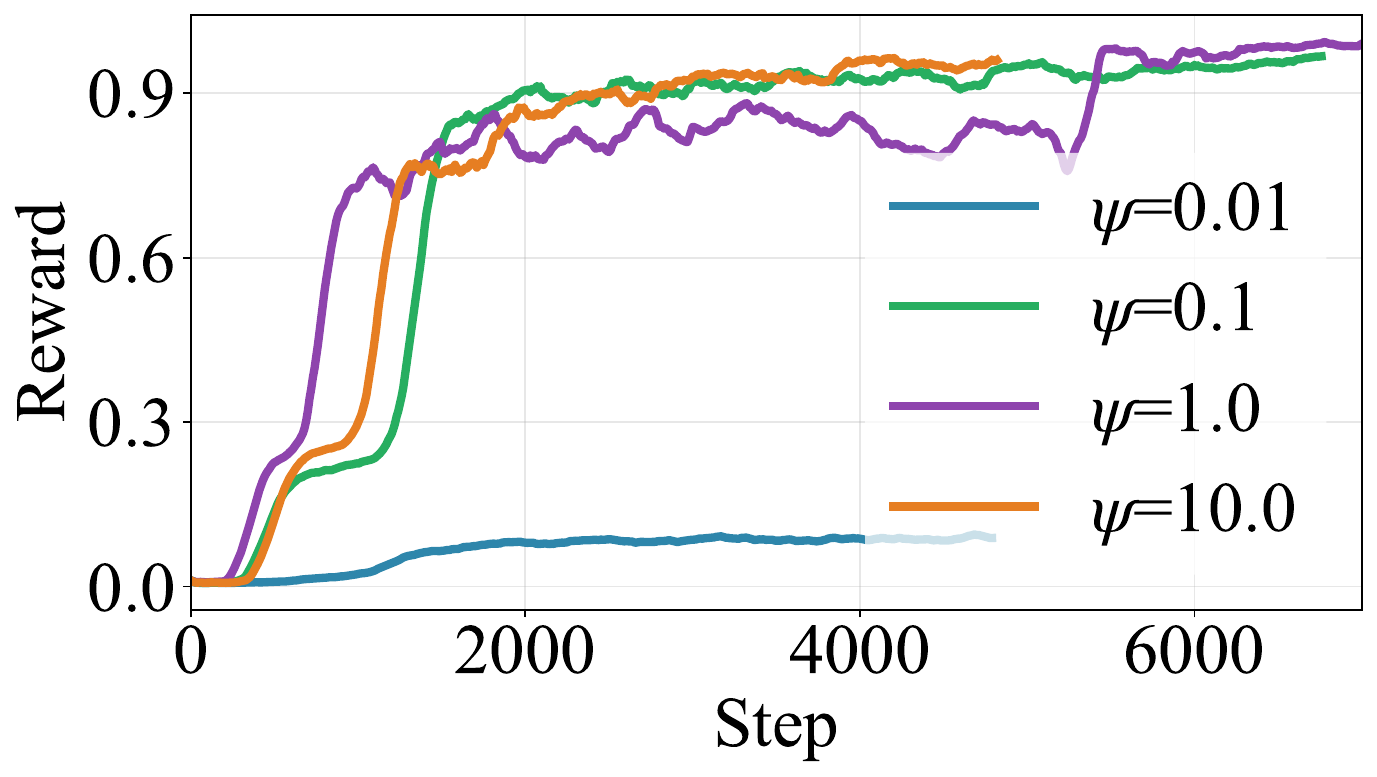}
    \end{subfigure}
    \hfill
    \begin{subfigure}{0.32\textwidth}
        \centering
        \includegraphics[width=\linewidth]{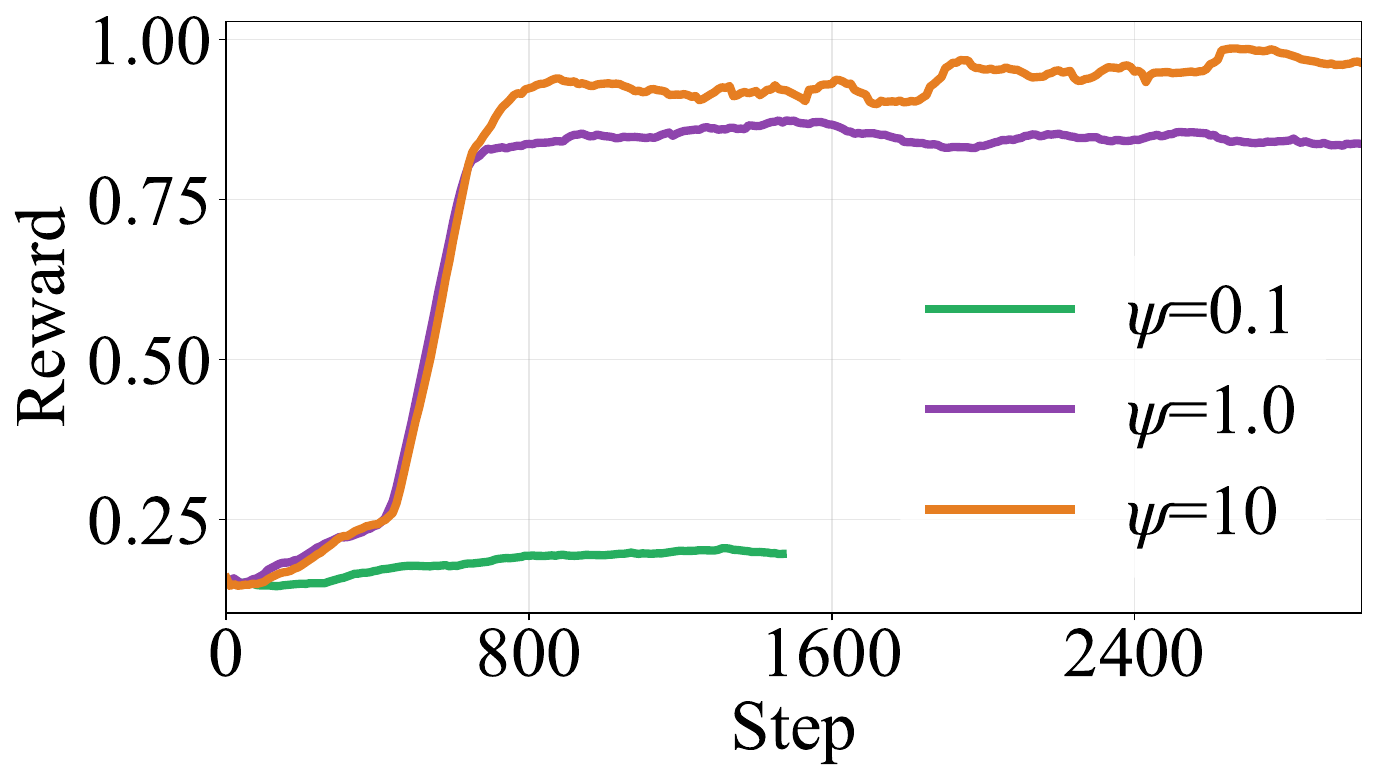}
    \end{subfigure}  
    \caption{Ablation study on Token-level Logit Centralization (TLC) and the energy-guidance coefficient $\psi$. \textbf{Left}:Training dynamics on Countdown with different models with and without TLC; \textbf{Mid:} Training dynamics on Sudoku with Dream-7B.; \textbf{Right:} Training dynamics on Sudoku with LLaDA-8B. TLC demonstrate a more stable training behavior. The improved training reward by increasing the temperature $\psi$ implies the effectiveness of GDSD.}
    \label{fig:aba_study}
    \vspace{-1.5em}
\end{figure}

%% file: neurips2026/appendix.tex
\appendix


\input{neurips2026/appendix/other_checkpoints}

%% file: neurips2026/appendix/other_checkpoints.tex
\newpage
\section{Proof}

\subsection{Lemma: Token-level Logit Centralization}
\label{sec:proof_lc}
\begin{lemma}
Denote the n-th token of clean completion $x_0$ as $x^{(n)}_0$, for any MDM $p$, we have the token-level logit centralization defined in \Cref{eq:def_lc} equivalent to sequence-level centralized logit $\log p(x_0|x_t) - \sum_{x_0} \log p(x_0|x_t)$ :
\begin{align}
\log p(x_0|x_t) - \sum_{x_0} \log p(x_0|x_t) = \sum_n^{N} \log p(x^{(n)}_0|x_t) - \sum_n^{N} \frac{1}{|\mathcal{V}|} \sum_{x'^{(n)}_0 \in \mathcal{V}} \log p(x'^{(n)}_{0}|x_t)
\label{eq:practical_lc}
\end{align}
\end{lemma}
\label{lemma:tlc}
\begin{proof}
The Left-Hand Side (LHS) of Equation \Cref{eq:practical_lc} represents the centralized log-probability of the sequence $x_0$:$$\text{LHS} = \log p(x_0|x_t) - \frac{1}{|\mathcal{X}|} \sum_{x_0 \in \mathcal{X}} \log p(x_0|x_t)$$Using the factorization property $\log p(x_0|x_t) = \sum_{n} \log p(x_0^{(n)}|x_t)$, we substitute this into both terms:$$\text{LHS} = \sum_{n=1}^N \log p(x_0^{(n)}|x_t) - \frac{1}{|\mathcal{X}|} \sum_{x_0 \in \mathcal{X}} \left( \sum_{n=1}^N \log p(x_0^{(n)}|x_t) \right).$$

The space of all possible sequences $\mathcal{X}$ is the Cartesian product of the vocabulary $\mathcal{V}$ for each position $n$: $\mathcal{X} = \mathcal{V} \times \mathcal{V} \times \dots \times \mathcal{V}$. The total number of sequences is $|\mathcal{X}| = |\mathcal{V}|^N$. We can rewrite the sum over $\mathcal{X}$ as a nested sum over each token position:$$\sum_{x_0 \in \mathcal{X}} \sum_{n=1}^N \log p(x_0^{(n)}|x_t) = \sum_{n=1}^N \left( \sum_{x_0 \in \mathcal{X}} \log p(x_0^{(n)}|x_t) \right).$$

Then for a fixed position $n$, the term $\log p(x_0^{(n)}|x_t)$ only depends on the token at that specific position. In the sum over all $|\mathcal{V}|^N$ possible sequences, each specific token $v \in \mathcal{V}$ appears at position $n$ exactly $|\mathcal{V}|^{N-1}$ times. Therefore:$$\sum_{x_0 \in \mathcal{X}} \log p(x_0^{(n)}|x_t) = |\mathcal{V}|^{N-1} \sum_{v \in \mathcal{V}} \log p(v|x_t).$$

Substitute this back into the LHS equation:$$\text{LHS} = \sum_{n=1}^N \log p(x_0^{(n)}|x_t) - \frac{1}{|\mathcal{V}|^N} \sum_{n=1}^N \left( |\mathcal{V}|^{N-1} \sum_{v \in \mathcal{V}} \log p(v|x_t) \right)$$Simplify the fraction $\frac{|\mathcal{V}|^{N-1}}{|\mathcal{V}|^N} = \frac{1}{|\mathcal{V}|}$:$$\text{LHS} = \sum_{n=1}^N \log p(x_0^{(n)}|x_t) - \sum_{n=1}^N \frac{1}{|\mathcal{V}|} \sum_{v \in \mathcal{V}} \log p(v|x_t).$$

By swapping the index $v$ for $x_0'^{(n)}$ to match the notation in the image:$$\text{LHS} = \sum_{n=1}^N \log p(x_0^{(n)}|x_t) - \sum_{n=1}^N \frac{1}{|\mathcal{V}|} \sum_{x_0'^{(n)} \in \mathcal{V}} \log p(x_0'^{(n)}|x_t)$$This matches the RHS of Equation \Cref{eq:practical_lc} exactly. 

By the factorization of the sequence-level denoising distribution, the target model satisfies
\(
p^*(x_0 | x_t) = \prod_{n=1}^N p^*(x_0^{(n)} | x_t),
\)
where
\(
p^*(x_0^{(n)} | x_t) = \mathrm{SoftMax} \big(\log \bar{p}^*(x_0^{(n)} | x_t)\big).
\)
It follows that \algname with token-level logit centralization has the same optimal probability distribution as the original \algname formulation.
\end{proof}

\subsection{Proof to Proposition \ref{proposition:centering}}
\label{sec:proof_centering}
\begin{proof}
According to Lemma \ref{lemma:tlc}, token-level logit centralization (TLC) is equivalent to sequence-level centralization. Therefore, the centralized logit of the teacher denoiser becomes 
\begin{align}
& \log \bar p^*(x_0|x_t) \\
 = &\log p^*(x_0|x_t) - \frac{1}{|\mathcal{X}|} \sum_{x_0} \log p^*(x_0|x_t)  \\
=  &\beta \log p_\text{old}(x_0|x_t) + (1-\beta) \log p_\text{ref}(x_0|x_t) + \cancel{\log \sum_{x_0} p_{\text{old}}(x_0|x_t)^{1-\beta} p_{\text{ref}}(x_0|x_t)^{\beta}} \nonumber  \\
&+ \psi A(x_0) -  \textcolor{red}{\cancel{A_t(x_t)}} \\
- &\Big(\frac{1}{|\mathcal{X}|} \sum_{x_0} \beta \log p_\text{old}(x_0|x_t) + (1-\beta) \log p_\text{ref}(x_0|x_t) +\cancel{\log \sum_{x_0} p_{\text{old}}(x_0|x_t)^{1-\beta} p_{\text{ref}}(x_0|x_t)^{\beta}} \nonumber  \\
& + \psi {A}(x_0) - \textcolor{red}{\cancel{A_t(x_t)}}\Big) \\
= & \beta \log p_\text{old}(x_0|x_t) + (1-\beta) \log p_\text{ref}(x_0|x_t) - \frac{1}{|\mathcal{X}|} \sum_{x_0} \big( \beta \log p_\text{old}(x_0|x_t) + (1-\beta) \log p_\text{ref}(x_0|x_t)\big) \nonumber \\
 &+ \psi A(x_0) - \sum_{x_0} A(x_0) \\
 = &  \beta \log \bar p_\text{old}(x_0|x_t) + (1-\beta) \log \bar  p_\text{ref}(x_0|x_t) + \psi A(x_0) - \sum_{x_0} A(x_0).
\end{align}
The first equality follows from the definition of centralized logits, which subtracts
the average logit over $\mathcal X$. We then substitute the log-form of
$p^*(x_0 \mid x_t)$ into both the original logit and its average. The geometric-mixture normalizer and $A_t(x_t)$ depend only on $x_t$, so they are
constant with respect to $x_0$ and cancel after centralization. The remaining
old-policy and reference-policy terms can then be regrouped into
$\log \bar p_\text{old}(x_0 \mid x_t)$ and
$\log \bar p_\text{ref}(x_0 \mid x_t)$, while the advantage term is centralized
by subtracting its average over $\mathcal X$.
\end{proof}


\subsection{Proof of Proposition \ref{prop:forward_kl_aw_elbo}}
\label{proof:forward_kl_aw_elbo}
\begin{proof}
Starting with the definition of the forward \text{KL} objective for a diffusion large language model (dLLM):
$$\mathcal{L}_{\text{fwd}}(\theta)=\mathbb{E}_{t\sim U(0,1),x_{t}\sim p_{t}^{*}}[D_{\text{KL}}(p^{*}(\cdot|x_{t})||p_{\theta}(\cdot|x_{t}))].$$ Expanding the Kullback–Leibler divergence term into its expectation form:
$$\mathcal{L}_{\text{fwd}}(\theta)=\mathbb{E}_{t\sim U(0,1),x_{t}\sim p_{t}^{*}}\mathbb{E}_{x_{0}\sim p^{*}(\cdot|x_{t})}[\log p^{*}(x_{0}|x_{t}) - \log p_{\theta}(x_{0}|x_{t})].$$ By substituting the definition of the energy-guided target denoising distribution $$p^*(x_0|x_t) = p_\text{old}^\text{ref}(x_0|x_t) \cdot \exp(\psi A(x_0) - A_t(x_t)),$$ the objective becomes:
$$\mathbb{E}_{t,x_{t}}\mathbb{E}_{x_{0}\sim p^{*}}[\log p_{\text{old}}^{\text{ref}}(x_{0}|x_{t}) + \psi A(x_{0}) - A_{t}(x_{t}) - \log p_{\theta}(x_{0}|x_{t})].$$ To find the gradient with respect to the model parameters $\theta$, we identify the terms independent of $\theta$:
\begin{itemize}
\item $\log p_{\text{old}}^{\text{ref}}(x_{0}|x_{t})$ depends only on the old and reference policies.
\item $\psi A(x_{0})$ is the advantage function based on the reward.
\item $A_{t}(x_{t})$ is the log-normalization constant (partition function).
\end{itemize}Since these terms have zero gradients w.r.t. $\theta$, the gradient of the objective simplifies to:
$$\nabla_{\theta}\mathcal{L}_{\text{fwd}}(\theta) = \nabla_{\theta}\mathbb{E}_{t\sim U(0,1),x_{t}\sim p_{t}^{*}}\mathbb{E}_{x_{0}\sim p^{*}(\cdot|x_{t})}[-\log p_{\theta}(x_{0}|x_{t})]$$Finally, by applying importance sampling to change the expectation from the target distribution $p^*$ to the reference distribution $\pi_{\text{old}}^{\text{ref}}$, we incorporate the advantage as a weight:
$$\nabla_{\theta}\mathcal{L}_{\text{fwd}}(\theta) \propto \nabla_{\theta}\mathbb{E}_{x_{0}\sim\pi_{\text{old}}^{\text{ref}}}[\exp(A(x_{0}))\mathbb{E}_{t\sim U(0,1),x_{t}\sim p_{t}^{*}}[-\log p_{\theta}(x_{0}|x_{t})]].$$ This result is precisely the objective for wd1 with ELBO as likelihood approximation and positive weight, namely Advantage-Weighted Discrete Cross-Entropy (AW-DCE) \citep{tang2025wd1} and Distribution Matching Policy Optimization (DMPO) \citep{zhu2025enhancing}.
\end{proof}

\subsection{Proof of Reverse-KL Distillation to Policy Gradient}
\label{proof:rkl_pg}
\begin{proof}
By the definition of the energy-guided teacher,
\begin{align}
\log p^*(x_0\mid x_t)
=
\log p_\text{old}^\text{ref}(x_0\mid x_t)
+
\psi A(x_0)
-
A_t(x_t).
\end{align}
Therefore,
\begin{align}
\mathcal{L}_{\text{rev}}(\theta)
&=
\mathbb{E}_{t,x_t\sim p_{\theta,t}}
\left[
D_{\mathrm{KL}}\big(p_\theta(\cdot\mid x_t)\,\|\,p^*(\cdot\mid x_t)\big)
\right] \\
&=
\mathbb{E}_{t,x_t\sim p_{\theta,t},\,x_0\sim p_\theta(\cdot\mid x_t)}
\left[
\log p_\theta(x_0\mid x_t)-\log p^*(x_0\mid x_t)
\right] \\
&=
\mathbb{E}_{t,x_t,x_0}
\left[
-\psi A(x_0)
+
A_t(x_t)
+
\log p_\theta(x_0\mid x_t)
-
\log p_\text{old}^\text{ref}(x_0\mid x_t)
\right] \\
&=
\mathbb{E}_{x_0\sim \pi_\theta}
\left[-\psi A(x_0)\right]
+
\mathbb{E}_{t\sim U(0,1),\,x_0\sim\pi_\theta,\,x_t\sim q(\cdot\mid x_0)}
\left[A_t(x_t)\right] \\
&\quad+
\mathbb{E}_{x_0\sim\pi_\theta}
\left[
L(x_0;p_\theta)-L(x_0;p_\text{old}^\text{ref})
\right],
\end{align}
where the last equality uses the sampling equivalence
$x_0\sim\pi_\theta,\ x_t\sim q(\cdot\mid x_0)$ and the definition
\[
L(x_0;p)
:=
\mathbb{E}_{t\sim U(0,1),\,x_t\sim q(\cdot\mid x_0)}
\left[-\log p(x_0\mid x_t)\right].
\]
\end{proof}

\section{Additional Related Work}
\label{append:related_work}

\textbf{Reinforcement Learning for Diffusion Models.} Reinforcement learning for diffusion models has been widely studied in the context of diffusion policies for robotics, with applications including planning, offline RL, imitation learning, and safe control~\citep{wang2022diffusion,janner2022planning,hansen2023idql,lu2023contrastive,chen2023score,ren2024diffusion,zhang2025energy,chi2025diffusion,lu2025makes,mcallister2025flow,cheng2025safe}. However, these works are primarily built upon continuous diffusion models and are typically evaluated on relatively small-scale policies. More recently, several works have explored RL fine-tuning for large-scale vision diffusion models~\citep{liu2025flow,zheng2025diffusionnft,choi2026rethinking}. These methods operate under continuous diffusion or flow-matching formulations, where the score function or velocity vector field is explicitly modeled for image generation. In contrast, discrete diffusion language models define denoising distributions or concrete score \citep{meng2022concrete} over tokens and do not expose analogous continuous score or velocity fields. This structural difference makes existing RL methods for continuous diffusion models difficult to apply directly to dLLMs, motivating RL objectives designed specifically for discrete denoising distributions.

\textbf{Reinforcement Learning via Logit Matching.}
Various post-training methods for autoregressive models can be interpreted through the lens of logit matching. For example, Kimi performs RL by matching the tractable policy to an exponential-reward-weighted old policy~\citep{team2025kimi}. Self-play policy optimization methods adopt a similar form in preference optimization~\citep{wu2024self,tang2025rspo}. These methods can be viewed as exponential-weight updates over the policy~\citep{freund1997decision}, where the weighting signal is the reward or advantage of a given response. They are also closely related to soft policy iteration methods, such as SAC~\citep{haarnoja2018soft}. However, these approaches assume access to a tractable response likelihood, which is unavailable in diffusion language models. More recently, diffusion RL methods have adopted logit-matching-style objectives. For continuous diffusion models, PAR~\citep{choi2026rethinking} shows that logit matching is stable and effective for RL post-training. EMBR~\citep{shankar2026energy} proposes a contrastive energy-matching objective for pairwise preference learning with diffusion language models, but it is derived from an offline DPO-style objective~\citep{rafailov2023direct} rather than online RL. In contrast, we derive GDSD from online objective PPO with reverse-KL regularization, analyze challenges of online RL, including TIM and normalization constant, and address them accordingly.

\textbf{Exact-Likelihood Diffusion RL.}
A few RL methods for dLLMs are based on trajectory likelihood. Unlike ELBO-based objectives, trajectory likelihood follows the actual decoding process and therefore provides an exact likelihood under the sampling distribution. A key ingredient in computing this likelihood is the token generation order. LLaDOU introduces an additional head to predict which tokens should be unmasked and trains this mechanism end-to-end with RL~\citep{huang2025reinforcing}. Similar ideas have also been explored for accelerating generation~\citep{chen2025dultra}. It is also possible to use trajectory likelihood without modifying the model architecture. For example, TraceRL applies reverse-process transition probabilities for policy optimization~\citep{wang2025revolutionizing}. However, trajectory-likelihood methods are computationally expensive, since the RL objective must be accumulated over diffusion steps, and current dLLMs predict the full sequence at each step. To reduce this cost, TraceRL and d2~\citep{wang2025d2} introduce step-merging strategies that trade off likelihood accuracy for efficiency. Therefore, although trajectory-likelihood methods avoid the bias of ELBO-based likelihood surrogates, they often require architectural changes or incur high computational cost as the number of diffusion steps increases. Moreover, their objective differs from the ELBO used in pre-training. In contrast, we aim to develop an RL method that retains the efficiency and pre-training compatibility of random-masking ELBO objectives, while avoiding the bias introduced by using ELBO as a likelihood surrogate.

\textbf{Likelihood-free Diffusion RL.} DiffusionNFT~\citep{zheng2025diffusionnft} proposes likelihood-free RL for continuous diffusion post-training by reformulating RL as reinforcement guidance and optimizing a flow-matching objective with positive and negative generations. While this avoids trajectory-likelihood estimation, it relies on the continuous diffusion/flow formulation and its velocity-field training target. In contrast, dLLMs define discrete token-denoising distributions, where such velocity fields are unavailable. LFPO~\citep{wei2026lfpo} extends the idea of DiffusionNFT to RL post-training for dLLMs. However, LFPO is essentially a reward-weighted cross-entropy/ELBO method~\citep{wei2026lfpo}, closely related to the forward-KL distillation instance discussed in \Cref{prop:forward_kl_aw_elbo}. To mitigate the data inefficiency caused by down-weighting negative samples, LFPO follows DiffusionNFT by constructing a negative dataset for explicit negative-response penalization. In contrast, GDSD performs logit matching on reward-guided denoisers, allowing both positive and negative advantage signals to be incorporated directly without relying on ELBO weighting or auxiliary negative-sample penalties.

\section{Additional Details}

\subsection{Algorithm}
\Cref{alg:gdsd} summarizes Guided Denoiser Self-Distillation (GDSD). At each training iteration, GDSD first samples a batch of prompts and uses the old policy to generate multiple completions through iterative re-masking decoding. These completions are scored by the reward function, and their advantages are computed within each prompt group by subtracting the group mean reward.GDSD then converts these reward signals into denoising-level supervision. For each sampled completion, it randomly samples diffusion times and applies the same forward masking process to construct masked states. The current, old, and reference models are evaluated on these masked inputs to obtain their denoising logits. Token-level logit centralization is applied to remove irrelevant logit offsets and eliminate the need to estimate the intractable normalization constant of the energy-guided teacher. Finally, GDSD constructs a guided teacher denoiser by combining the centralized old and reference logits with the advantage signal, and updates the current model by matching its centralized logits to this teacher. The old policy is periodically refreshed from the current model. Overall, GDSD turns RL into a denoiser self-distillation: rewards guide the teacher construction, while training remains a stable logit-matching problem instead of relying on ELBO-based likelihood estimation.

\begin{algorithm}[t]
\caption{Guided Denoiser Self-Distillation (GDSD)}
\label{alg:gdsd}
\begin{algorithmic}[1]
\REQUIRE Initial dLLM $\pi_\theta$, with reference policy $\pi_{\text{ref}}$, prompt dataset $\mathcal{D}$, reward function $r$, hyperparameters $\beta, \psi$, time-step samples $K$
\STATE Initialize $\pi_{\text{old}} \leftarrow \pi_\theta$
\FOR{each training iteration}
    \STATE Sample a batch of prompts $\{c^{(i)}\}_{i=1}^B$ from $\mathcal{D}$
    \STATE \textcolor{blue!40}{\textit{// Sampling stage}}
    \FOR{each prompt $c^{(i)}$}
        \STATE Generate $G$ completions $\{x_0^{(i,g)}\}_{g=1}^G \sim \pi_{\text{old}}(\cdot | c^{(i)})$ via iterative re-masking
        \STATE Compute rewards $r(x_0^{(i,g)})$ and advantages $A(x_0^{(i,g)}) = r(x_0^{(i,g)}) - \mathrm{mean}(r(\cdot))$
    \ENDFOR
    \STATE \textcolor{blue!40}{\textit{// Training stage}}
    \FOR{each completion $x_0$ with advantage $A(x_0)$}
        \STATE Sample $K$ time steps $\{t_k\}_{k=1}^K \sim U(0,1)$
        \FOR{$k = 1, \ldots, K$}
            \STATE Apply consistent forward process: $x_{t_k} \sim q(\cdot | x_0)$
        \ENDFOR
        \STATE Compute denoising log-probabilities $\{\log p(x_0 | x_{t_k})\}_{k=1}^K$, for all models $p=p_\theta, p_\text{old}, p_\text{ref}$.
        \STATE Apply token-level logit centralization (\Cref{eq:def_lc}) on denoising log-probabilities.
        \STATE Construct teacher logits: $\log \bar{p}^*(x_0 | x_{t_k}) = (1-\beta) \log \bar{p}_{\text{old}}(x_0 | x_{t_k}) + \beta \log \bar{p}_{\text{ref}}(x_0 | x_{t_k}) + \psi A(x_0)$
        \STATE Compute GDSD loss: $\mathcal{L}_{\text{GDSD}} = \frac{1}{K}\sum_{k=1}^K \big[\log \bar{p}_\theta(x_0 | x_{t_k}) - \log \bar{p}^*(x_0 | x_{t_k})\big]^2$
    \ENDFOR
    \STATE Update $\theta$ via gradient descent on $\mathcal{L}_{\text{GDSD}}$
    \STATE Update old policy: $\pi_{\text{old}} \leftarrow \pi_\theta$ every $\mu$ iterations
\ENDFOR
\end{algorithmic}
\end{algorithm}


\subsection{Additional Practical Design}
\label{sec:practical_design}

In this section, we present the final practical objective for \algfullname (\algname), incorporating key design choices that optimize computational efficiency. We introduce the practical objective and provide detailed explanations as follows:
\begin{myframe}[title=\algfullname (\algname) (practical)]
    \begin{align}
\mathbb{E}_{{t \sim U(0,1),\ x_0 \sim \pi^{\text{rm}}_{\text{old}},\ x_t \sim q(\cdot|x_0)}}\Bigg[ \Big( \log \frac{\bar{p}_\theta(x_0|x_t)}{\bar{p}_{\text{old}}(x_0|x_t)} - \psi {A}(x_0) \Big)^2 +  \beta \Big( \log \frac{\bar{p}_\theta(x_0|x_t)}{\bar{p}_\text{ref}(x_0|x_t)}\Big)^2 \Bigg].
\label{eq:egdd_practical}
\end{align}
\end{myframe}

\textbf{External Regularization.} To facilitate hyperparameter tuning, the reference model's log-probability can be decoupled from the primary Mean Squared Error (MSE) term. By treating it as an external regularization term $R$—similar to standard RLHF pipelines—we can compute the gradient as:
\begin{align}
\nabla_\theta \bar{\mathcal{L}}_\text{\algname} \propto \nabla_\theta  \mathbb{E}_{t \sim U(0,1),\ x_0 \sim \pi^{\text{rm}}_{\text{old}},\ x_t \sim q(\cdot|x_0)}\Big[ \big(  \log \frac{\bar{p}_\theta(x_0|x_t)}{\bar{p}_\text{old}(x_0|x_t)} - \psi A(x_0) \big)^2  + \colorbox{cyan!40}{$R$} \Big],
\end{align}
Where $R=\frac{\beta}{1-\beta}(\log\bar{p}_{\theta}(x_{0}|x_{t})-\log\bar{p}_{\text{ref}}(x_{0}|x_{t}))^{2}$. This formulation establishes a k2 KL divergence approximation \citep{schulman2020approximating}, allowing for a cleaner separation between the advantage-guided update and the reference model constraint to better control the balance between them.

\textbf{Single Inference in Training.} Obtaining denoising probabilities (e.g. $\log p_\theta(x_0|x_t)$) at different steps $t$ requires model inference for parametrized policies. This computation leads to significant overhead if denoising probabilities at different time step $t$ are computed via \textit{separate} model inferences. For efficiency, we batch $x_t$ at all sampled steps $t$, and compute the denoising probabilities with \textit{single} model inference. For other models (old and reference model), we apply the same design.

\textbf{Variance-Reduced and Re-weighted Logits.} We follow prior work~\citep{zhu2025llada,ou2025principled,gong2025diffucoder} and apply shared coupled sampling for masked sequences. Specifically, we merge the denoising logits obtained from a masked sequence with the one obtained from its complementary mask to reduce variance. We further reweight the logits using a $1/t$ schedule, which emphasizes updates near the clean sequence. Both designs lead to empirical improvements in our experiments.

\section{Additional Experimental Details}

\subsection{Histogram of Accuracy}
We summarize the test accuracy of our methods compared to existing ELBO-based methods in \Cref{fig:llada_acc}. Our methods GDSD consistently outperform baselines, particularly significant on Dream-7B base models.

\begin{figure}[h]
    \centering
    \includegraphics[width=.49\linewidth]{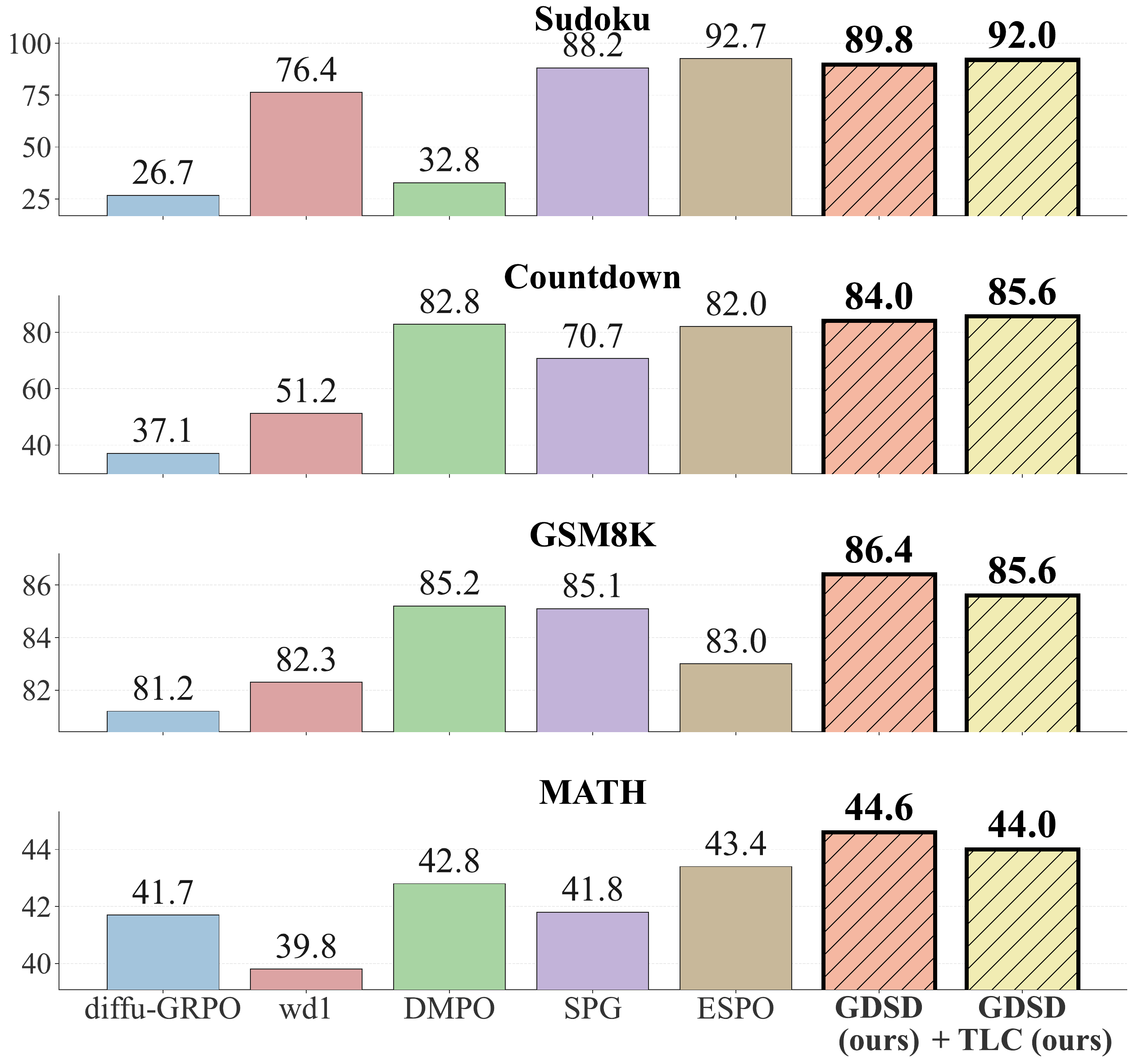}
    \includegraphics[width=.49\linewidth]{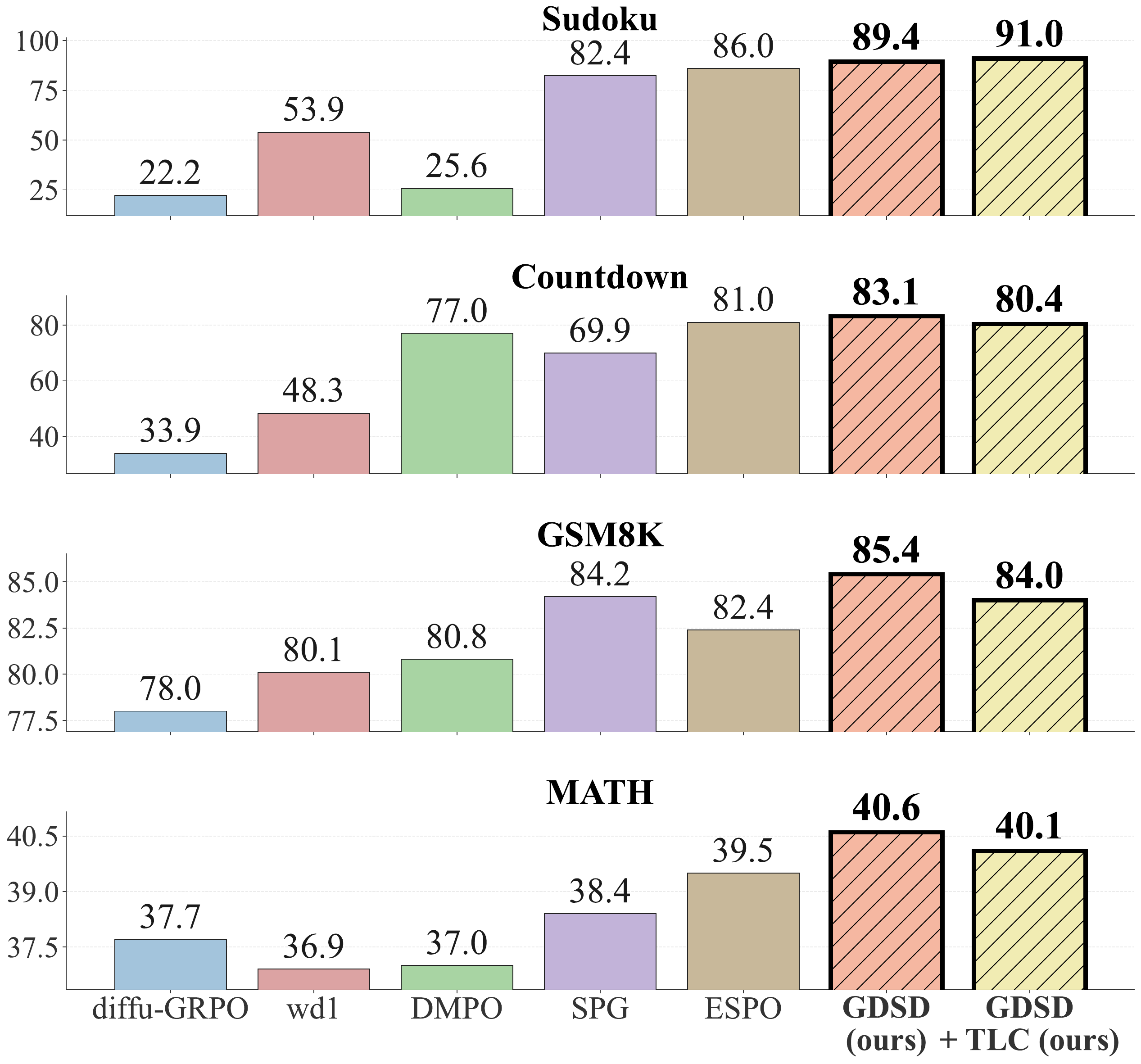}
    \caption{\textbf{Left:} Test accuracy of different methods (best across generation length 128, 256, and 512) on base model LLaDA-8B. \textbf{Right:} Average test accuracy of different methods across generation lengths, on base model LLaDA-8B.}
    \label{fig:llada_acc}
\end{figure}

\input{neurips2026/contents/reward_tlc}

\subsection{Coding Benchmarks}

In this section, we provide additional details and additional results on coding benchmarks.

\input{neurips2026/contents/code_results}

\subsection{Experimental Setups}
\label{append:exp_setup}

We conduct RL fine-tuning based on open-sourced dLLMs, LLaDA-8B-Instruct~\citep{nie2025large} and Dream-v0-Instruct-7B~\citep{ye2025dream}. We experiment on six benchmarks: GSM8K~\citep{DBLP:journals/corr/abs-2110-14168} and
MATH500~\citep{DBLP:conf/iclr/LightmanKBEBLLS24} for mathematical reasoning, Countdown~\citep{tinyzero} and
Sudoku~\citep{arel_sudoku} for logical reasoning, HumanEval~\citep{chen2021evaluatinglargelanguagemodels} and MBPP~\citep{austin2021programsynthesislargelanguage} for coding. The training configurations of logical planning (Countdown and Sudoku) and coding (HumanEval and MBPP) tasks follow the zero-shot setting as in ESPO~\citep{ou2025principled}. While for mathematical reasoning tasks, we found the format reward as d1~\citep{zhao2025d1}, wd1~\citep{tang2025wd1}, and SPG~\citep{wang2025spg} assists in RL performance, thus we follow the reward design in SPG~\citep{wang2025spg}. For coding tasks, we train the base model on AceCoder-87K~\citep{zeng2025acecoderacingcoderrl}.

\paragraph{Evaluation Setup.} We found that the evaluation protocol varies significantly across previous literature, so we aim to unify the evaluation as much as possible for a fair comparison. For logical planning tasks (Sudoku and Countdown), we follow the zero-shot evaluation in~\citep{ou2025principled} and test on generation lengths of 128, 256, and 512. The denoising steps are set to be half of the sequence length. For other tasks, the
number of denoising steps is set equal to the sequence length to improve performance.

\textbf{Remark.} We wish to clarify different evaluation settings in various papers for the future benefit of the community. In terms of reasoning tasks and coding tasks, ESPO respectively evaluates LLaDA with OpenCompass and Dream with lmeval, and they set the diffusion steps as $1\times$ generation length. Meanwhile, other papers like d1~\citep{zhao2025d1}, wd1~\citep{tang2025wd1}, SPG~\citep{wang2025spg} follow a similar evaluation with the codes provided in d1~\citep{zhao2025d1} and set the diffusion steps as $0.5\times$ generation length. As for the planning tasks, both pipelines set the diffusion steps as $0.5\times$ generation length, but d1, wd1, and ESPO follow a 0-shot evaluation for Sudoku, and SPG shuffles the dataset and tests in a 3-shot setting. For fair comparison, we fix zero-shot settings for Sudoku, Countdown, GSM8k, Math500, and humenval and 3-shot settings for MBPP. We fixed the diffusion steps as $1\times$ generation length and evaluated using lmeval. For results previously reported as incompatible with such a setting, we reproduce the results either from the checkpoints given (ESPO) or re-implement the training (SPG). We summarize the evaluation settings as in~\cref{tab:training_configurations}

\input{neurips2026/contents/hyperparameters}

\paragraph{Reward Setups.} We follow SPG~\citep{wang2025spg} for the reward function design, which encourages both correct answers and proper formatting. We provide details as follows.

\textbf{GSM8K.} We follow the Unsloth reward setup\footnote{\url{https://unsloth.ai/blog/r1-reasoning}}: 
\begin{itemize}
    \item XML Structure: +0.125 per correct formatting tag; small penalties for overlong contents after the closing tag.
    \item Formatting: 
    
    (Soft) +0.5 for outputs that have the following content: 
    
    \texttt{<reasoning>...</reasoning><answer>...</answer>}. 
    
    (Strict) +0.5 for exact formatting.
    \item Validity: +0.5 if the answer is valid (an integer).
    \item Correctness: +2.0 if the answer is correct.
\end{itemize}

\textbf{MATH500.} A format reward and a correctness reward are used:
\begin{itemize}
    \item Formatting: 1.00 if <answer></answer> tags are present with \textbackslash boxed inside; 0.75 if answer tags are present without \textbackslash boxed; 0.50 if answer tags are not present but \textbackslash boxed is present; 0.25 if neither the answer tags nor \textbackslash boxed is present.
    \item Correctness: 2.00 if the answer in \texttt{\textbackslash boxed\{\}} is correct.
\end{itemize}

\textbf{Countdown.} The reward covers three cases: +1.0 if the expression reaches the target using the exact numbers; +0.1 if the numbers are correct but does not reach the target; +0.0 otherwise.

\textbf{Sudoku.} The reward is based on cell-level matching against the ground-truth:
\begin{itemize}
    \item Answer Extraction: The model is expected to output the solution with \texttt{<answer>...</answer>} tags, and the digits inside the last answer tag are extracted.
    \item Length Normalization: The extracted answer is normalized to length 16 for the $4\times4$ Sudoku puzzle. If it is shorter than 16 digits, it is padded with zeros; if it is longer, it is truncated.
    \item Correctness: The reward is the fraction of cells compared to originally empty ones that match the ground-truth solution:
    \[
        r = \frac{\#\{\text{correctly filled empty cells}\}}
        {\#\{\text{empty cells}\}}.
    \]
    \item Missing Answer: If no valid answer is extracted, the reward is 0.0.
\end{itemize}

\textbf{Coding.} The reward consists of a format reward and an execution reward:
\begin{itemize}
    \item Formatting: The model output must contain a Markdown code block in the target language, e.g.,
    \texttt{```python ... '''}. The format reward is 1.0 if the code block is formatted correctly and the syntax is valid; 0.5 if correct code block formatting but with a syntax error; 0.0 otherwise.
    \item Execution: Completions with format reward 1.0 are executed. Each generated program is concatenated with the provided test cases and run in a sandboxed execution environment.
    \item Correctness: The execution reward is the test pass rate:
    \[
        r = \frac{\#\{\text{passed test cases}\}}
        {\#\{\text{total test cases}\}}.
    \]
    \item Invalid Code: If the completion fails the format check or has invalid syntax, it is not executed and receives execution reward 0.0.
\end{itemize}

\subsection{Hyperparameter Settings and Implementation Details}
\label{app:exp_hyper}
We follow ESPO~\citep{ou2025principled} for most hyperparameter settings. LoRA (Low-Rank Adaptation) (LoRA) with rank $r=128$  and scaling factor $\alpha=64$ is adopted for training.

For RL rollout, we use a sequence length of 256 tokens, and 128 diffusion steps. We employ confidence-based semi-autoregressive generation with block size 32, and set the temperature as 0.9. We set the number of completions per prompt $g$ as 6, and the number of Monte Carlo estimation samples $m$ as 2 due to computational constraints. Since the rollout stage dominates the training time, the average time per gradient update step for is similar to that of the other baselines. 

We train 3K steps (i.e., number of gradient updates) for GSM8K and MATH500, 2K steps for coding, 5K steps for Sudoku and 10K for countdown. For all RL models, we run evaluation every 100 steps and report the result of the checkpoint with the highest average accuracy.


\subsection{Additional Results Analysis}

\paragraph{Ablation study on Token Logits Centralization.} We plot the reward curve with (w/) and without (w/o) Token Logits Centralization (TLC) on several datasets in~\Cref{fig:tlc_reward_curves}. Results are obtained with all other hyperparameters fixed. Specifically, the curve shows that TLC consistently improves both the final attainable reward and the stability of the optimization process. Across GSM8K, Countdown, and Sudoku, adding TLC leads to a higher and more stable reward plateau, indicating that TLC helps the model obtain reward more effectively during likelihood-free optimization. We summarize the behaviour analysis below:

\begin{itemize}
    \item \textbf{TLC reduces unstable reward optimization.} In the experiment of GSM8k (\Cref{fig:tlc_reward_gsm8k}),  the model w/o TLC reaches a competitive reward peak at an early stage, but it fails to remain in this high-reward region and later exhibits a noticeable degradation. In contrast, the TLC curve quickly rises and then stays around a high reward level, suggesting that TLC stabilizes optimization after the model discovers rewarding behaviors. A similar pattern was found in Countdown experiments (\Cref{fig:tlc_reward_ctd}): without TLC, the reward curves show clear oscillations and repeated upward/downward fluctuations, while TLC substantially stabilizes the training dynamics and improves the final reward for both LLaDA and Dream. 
    \item \textbf{TLC potentially escapes the model from local minima.} We found an interesting pattern in the reward curve for Sudoku experiments (\Cref{fig:tlc_reward_sudoku}): At the beginning, the model with TLC does not immediately achieve the highest reward and stays around a moderate reward level, roughly around $0.8$. However, after approximately 4k training steps, the TLC undergoes a sharp reward improvement and eventually surpasses the non-TLC baseline. Our checkpoint analysis confirms that this transition corresponds to a substantial improvement in evaluation accuracy: before the jump, the model accuracy is only around $50\%$, while after the jump it increases to around $90\%$, compared with roughly $80\%$ for the model trained without TLC. This suggests that the non-TLC run may converge to a suboptimal local basin, whereas TLC helps the model escape this basin and continue improving. 
\end{itemize}

We believe these empirical results can further support our theoretical derivation: TLC provides an unbiased RL optimization, enabling more stable and effective optimization over ELBO-based and likelihood-based methods.

\textbf{Training Reward Dynamics.} According to the reward dynamics in \Cref{fig:reward_curves} and \Cref{fig:aba_study}, GDSD exhibits more stable and effective training than previous ELBO-based methods. TLC further shows slightly better convergence than direct matching in GDSD. In contrast, SPG on Countdown and ESPO on coding exhibit poor training performance.



\section{Others}
\subsection{Compute Resources}
Experiments are conducted on 8 GPUs, including AMD Instinct MI308X and A800.

\subsection{Broader Impact}

This work studies efficient reinforcement learning for diffusion large language models. By avoiding likelihood estimation and reverse-chain policy gradients, the proposed method can reduce the computational cost and instability of post-training, making alignment and reasoning-oriented training more accessible.

However, stronger and cheaper post-training may also amplify risks associated with language models, including misinformation, spam, academic dishonesty, automated persuasion, and harmful content generation. Since our method relies on reward signals, misspecified rewards may reinforce biases, unsafe behaviors, or undesirable shortcuts.

Models trained with our method should therefore be evaluated beyond reasoning and coding benchmarks, including safety, robustness, bias, hallucination, and misuse-related evaluations. Deployment should be paired with safeguards such as data curation, red-teaming, monitoring, and access controls when appropriate.


%% file: neurips2026/contents/reward_tlc.tex
\begin{figure}[t]
    \centering
    \begin{subfigure}{0.32\textwidth}
        \centering
        \includegraphics[width=\linewidth]{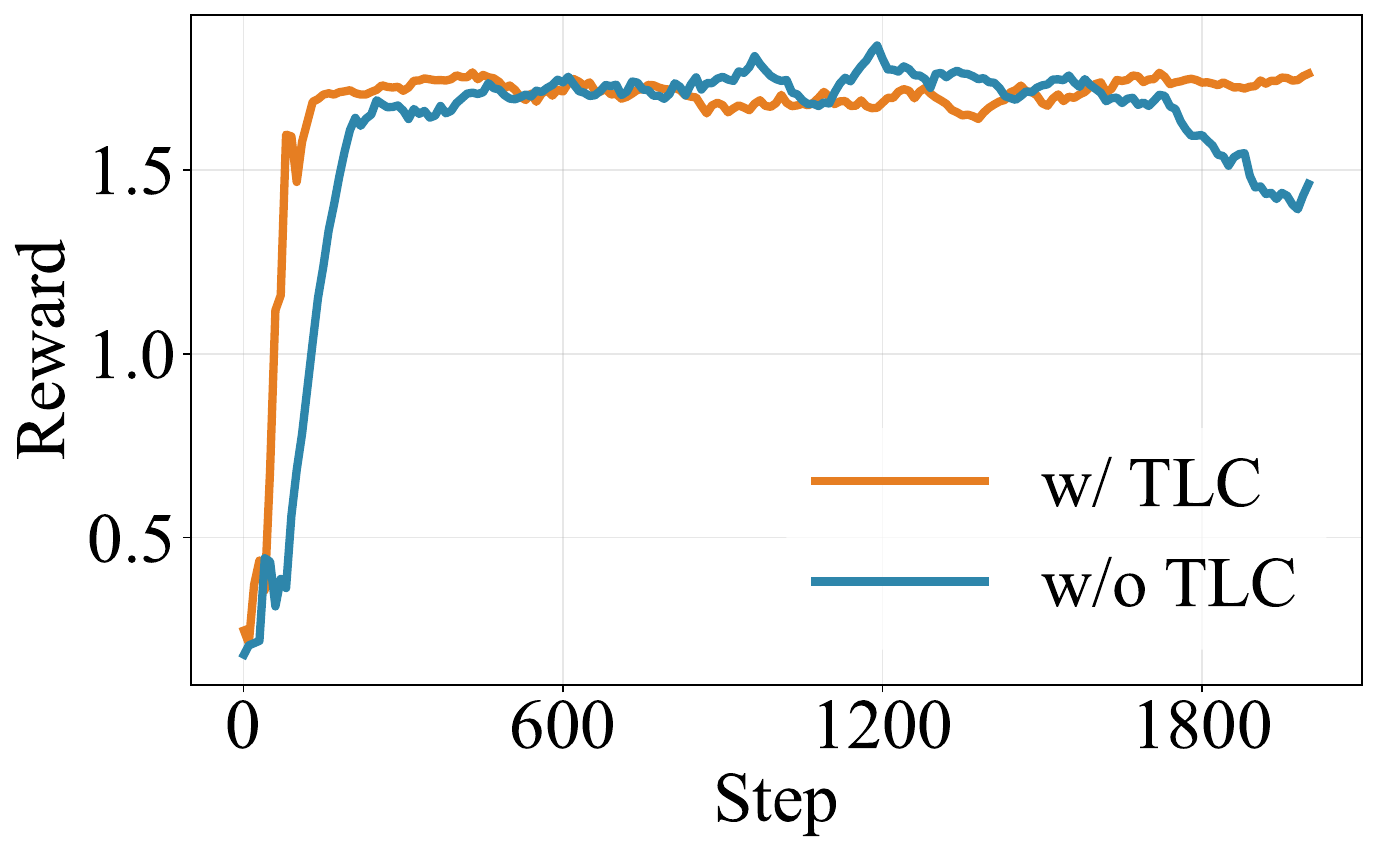}
        \caption{LLaDA-8B trained on GSM8k}
        \label{fig:tlc_reward_gsm8k}
    \end{subfigure}
    \hfill
    \begin{subfigure}{0.32\textwidth}
        \centering
        \includegraphics[width=\linewidth]{figures/reward_curves/tlc_abaltion/ctd_abatlc.pdf}
        \caption{Countdown}
        \label{fig:tlc_reward_ctd}
    \end{subfigure}
    \hfill
    \begin{subfigure}{0.32\textwidth}
        \centering
        \includegraphics[width=\linewidth]{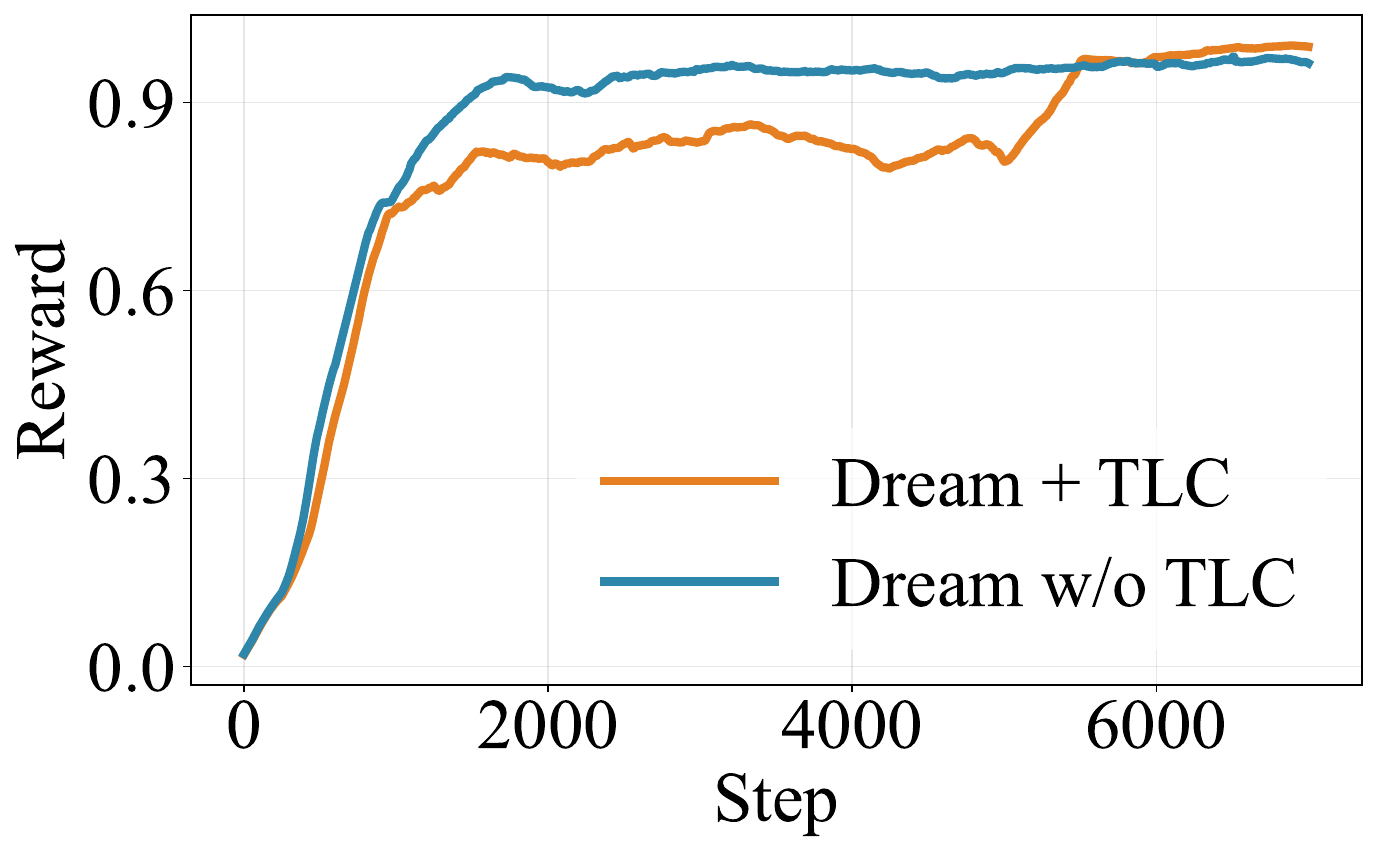}
        \caption{Sudoku}
        \label{fig:tlc_reward_sudoku}
    \end{subfigure}

    \caption{Reward dynamics of GDSD with and without token logits centralization (TLC). TLC consistently improves the stability of training process and leads to a higher reward level.}
    \label{fig:tlc_reward_curves}
\end{figure}

%% file: neurips2026/contents/code_results.tex
\begin{table}[h]
\caption{\textbf{Performance on coding benchmarks.} Results with $\dagger$ denotes the results re-evaluated by pass@1 with lm-eval~\citep{eval-harness} on the checkpoints from ~\citet{ou2025principled} or trained from scratch;  other baseline results are extracted from ~\citet{wang2025spg}. \algname generally improves the average performance over ELBO-based methods.}
\vspace{2mm}
\label{tab:llada_code}
\centering
\scriptsize
\setlength{\tabcolsep}{4pt}
\renewcommand{\arraystretch}{1.1}
\begin{tabular}{l ccc>{\columncolor{gray!20}}c ccc>{\columncolor{gray!20}}c ccc>{\columncolor{gray!20}}c}
\toprule
& \multicolumn{8}{c}{\textbf{HumanEval(0)}} & \multicolumn{4}{c}{\textbf{MBPP(3)}} \\
\cmidrule(lr){2-9} \cmidrule(lr){10-13}
& \multicolumn{4}{c}{\textbf{-}} & \multicolumn{4}{c}{\textbf{Plus}} & \multicolumn{4}{c}{\textbf{-}} \\
\cmidrule(lr){2-5} \cmidrule(lr){6-9} \cmidrule(lr){10-13}
\textbf{Model / Seq Len} & \textbf{128} & \textbf{256} & \textbf{512} & \textbf{Avg.}
                        & \textbf{128} & \textbf{256} & \textbf{512} & \textbf{Avg.}
                        & \textbf{128} & \textbf{256} & \textbf{512} & \textbf{Avg.} \\
\midrule
\textbf{LLaDA}
  & 28.1 & 35.4 & 34.8 & 32.8
  & 23.2 & 30.5 & 41.5 & 31.7
  & 36.2 & 42.0 & 38.1 & 38.8 \\
+ diffu-GRPO (d1)
  & 29.3 & 37.8 & 37.2 & 34.8
  & 22.0 & 29.9 & 37.2 & 29.7
  & 34.8 & 36.6 & 38.0 & 36.5 \\
+ wd1
  & 25.6 & 39.0 & 38.4 & 34.3
  & 29.9 & 29.9 & 32.9 & 30.9
  & 38.0 & 37.2 & 34.4 & 36.5 \\
+ SPG
  & 37.2$^{\dagger}$ & 41.5$^{\dagger}$ & 44.5$^{\dagger}$ & 41.1
  & 32.9$^{\dagger}$ & 34.2$^{\dagger}$ & 37.8$^{\dagger}$ & 35.0
  & 40.4$^{\dagger}$ & 40.8$^{\dagger}$ & 40.4$^{\dagger}$ & 40.5 \\
+ ESPO
  & 42.1$^{\dagger}$ & \textbf{48.2}$^{\dagger}$ & 41.4$^{\dagger}$ & 43.9
  & 24.4 & 36.6 & \textbf{42.7} & 34.6
  & \textbf{43.6}$^{\dagger}$ & 43.2$^{\dagger}$ & 41.2$^{\dagger}$ & 42.7 \\
\midrule
+ \textbf{\algname} (ours)
  & 42.1 & 45.1 & \textbf{45.7} & \textbf{44.3}
  & 36.0 & 38.4 & 41.5 & 38.6
  & 40.6 & 41.8 & \textbf{43.6} & 42.0 \\
+ \textbf{\algname w/ TLC} (ours) 
  & \textbf{43.3} & 43.9 & 43.3 & 43.5
  & \textbf{38.4} & \textbf{39.6} & 39.6 & \textbf{39.2}
  & 43.0 & \textbf{43.6} & 43.2 & \textbf{43.3} \\
\bottomrule
\end{tabular}

\end{table}

%% file: neurips2026/contents/hyperparameters.tex
\begin{table}[htbp]
\label{tab:training_configurations}
  \centering
  \caption{Training Hyperparameters across Different Tasks.}
  \label{tab:training_hyperparameters}
  \resizebox{1\textwidth}{!}{%
    \begin{tabular}{lcccccc}
      \toprule
      \textbf{Hyperparameter} & \textbf{Sudoku} & \textbf{Countdown} & \textbf{GSM8k} & \textbf{Math500} & \textbf{HumanEval} & \textbf{MBPP} \\
      \midrule
      \textbf{Task Category} & Planning & Planning & Reasoning & Reasoning & Coding & Coding \\
      \textbf{Num Iterations ($\mu$)} & 8 & 8 & 8 & 8 & 8 & 8 \\
      \textbf{Train Batch Size} & 10 & 10 & 10 & 8 & 3 & 3 \\
      \textbf{Generation Batch Size} & 8 & 8 & 6 & 8 & 10 & 10 \\
      \textbf{Num Generations} & \multicolumn{6}{c}{1$\times$ Number of GPUs} \\
      \textbf{Num MC ($K$)} & 2 & 2 & 2 & 2 & 4 & 4 \\
      \textbf{Gradient Accumulation} & 4 & 4 & 4 & 4 & 20 & 20 \\
      \textbf{Training Completion Length} & \multicolumn{6}{c}{256} \\
      \textbf{Max Prompt Length} & \multicolumn{6}{c}{400} \\
      \textbf{Training Diffusion Steps} & \multicolumn{6}{c}{128}  \\
      \textbf{Training Temperature} & \multicolumn{6}{c}{1.0} \\
      \midrule
      \textbf{Learning Rate} & 1e-5 & 1e-5 & 3e-6 & 1e-5 & 3e-6 & 3e-6 \\
      \textbf{Training Steps} & 5k & 10k & 3k & 3k & 2k & 2k \\
      \textbf{Learning Rate Scheduler} & \multicolumn{6}{c}{constant\_with\_warmup} \\
      \textbf{Warmup Ratio} & \multicolumn{6}{c}{0.001} \\
      \textbf{Weight Decay} & \multicolumn{6}{c}{0.01} \\
      \textbf{Max Grad Norm} & 0.2 & 0.2 & 0.2 & 0.2 & 0.8 & 0.8 \\
      \midrule
      \textbf{Beta (KL Coef)} & 1e-3 & 5e-4 & 1e-3 & 1e-4 & 5e-2 & 5e-2 \\
      \textbf{Psi ($\psi$)} & \multicolumn{6}{c}{10.0} \\
      \textbf{Epsilon (Clip)} & \multicolumn{6}{c}{0.2} \\
      \textbf{Scale Rewards} & \multicolumn{6}{c}{false} \\
      \midrule
      \textbf{LoRA Rank ($r$)} & \multicolumn{4}{c}{128} &\multicolumn{2}{c}{Full parameter tuning}\\
      \textbf{LoRA Alpha ($\alpha$)} & \multicolumn{4}{c}{64} &\multicolumn{2}{c}{Full parameter tuning}\\
      \textbf{LoRA Dropout} & \multicolumn{4}{c}{0.05} &\multicolumn{2}{c}{Full parameter tuning} \\
      \textbf{Use PEFT} & \multicolumn{4}{c}{true} & \multicolumn{2}{c}{false} \\
      \textbf{Gradient Checkpointing} & \multicolumn{4}{c}{false} & \multicolumn{2}{c}{true} \\
      \bottomrule
    \end{tabular}%
  }
\end{table}

\begin{table}[htbp]
  \centering
  \caption{Unified evaluation settings adopted for fair comparison across tasks and benchmarks.}
  \label{tab:eval_settings}
  \resizebox{1\textwidth}{!}{%
    \begin{tabular}{lcccccc}
      \toprule
      \textbf{Task} & \textbf{Sudoku} & \textbf{Countdown} & \textbf{GSM8k} & \textbf{Math500} & \textbf{HumanEval} & \textbf{MBPP} \\
      \midrule
      \textbf{Task Category} & Planning & Planning & Reasoning & Reasoning & Coding & Coding \\
      \textbf{Evaluation Source Code} & \texttt{d1} & \texttt{d1} & \texttt{lmeval} & \texttt{lmeval} & \texttt{lmeval} & \texttt{lmeval} \\
      \textbf{Num of Shots} & 0-shot & 0-shot & 0-shot & 0-shot & 0-shot & 3-shot \\
      \textbf{Generation Length} & \multicolumn{6}{c}{\{128,256,512\}} \\
      \textbf{Diffusion Steps} & \multicolumn{2}{c}{\{64,128,256\}} &\multicolumn{4}{c}{\{128,256,512\}} \\
      \bottomrule
    \end{tabular}%
  }
\end{table}